\documentclass{sage}

\ifx\definitionMacros\undefined\else \fi
 \ifx\optionkeymacros\undefined\else \fi

\catcode`\Œ=\active\defŒ{{\aa}}       
\catcode`\º=\active\defº{\int}        
\catcode`\=\active\def{\c c}        
\catcode`\¶=\active\def¶{\partial}    
\catcode`\Ä=\active\defÄ{\oint}       
\catcode`\Æ=\active\defÆ{\triangle}   
\catcode`\Â=\active\defÂ{\neg}        
\catcode`\µ=\active\defµ{\mu}         
\catcode`\¿=\active\def¿{{\o}}        
\catcode`\¹=\active\def¹{\pi}         
\catcode`\Ï=\active\defÏ{{\oe}}       
\catcode`\§=\active\def§{{\ss}}       
\catcode`\ =\active\def {\dagger}     
\catcode`\Ã=\active\defÃ{\sqrt}       
\catcode`\·=\active\def·{\Sigma}      
\catcode`\Å=\active\defÅ{\approx}     
\catcode`\½=\active\def½{\Omega}      
\catcode`\£=\active\def£{{\it\$}}     
\catcode`\°=\active\def°{\infty}      
\catcode`\¤=\active\def¤{{\S}}        
\catcode`\¦=\active\def¦{{\P}}        
\catcode`\¥=\active\def¥{\bullet}     
\catcode`\»=\active\def»{\leavevmode\raise.585ex\hbox{\b a}}      
\catcode`\¼=\active\def¼{\leavevmode\raise.6ex\hbox{\b o}}        
\catcode`\­=\active\def­{\not=}       
\catcode`\²=\active\def²{\leq}        
\catcode`\³=\active\def³{\geq}        
\catcode`\Ö=\active\defÖ{\div}        
\catcode`\É=\active\defÉ{{\dots}}     
\catcode`\¾=\active\def¾{{\ae}}       
\catcode`\Ç=\active\defÇ{\ll}         
\catcode`\Ò=\active\defÒ{``}          
\catcode`\Á=\active\defÁ{!`}          
\catcode`\¢=\active\def¢{\rlap/c}     
\catcode`\Ô=\active\defÔ{`}           
\catcode`\Õ=\active\defÕ{'}           

\catcode`\=\active\def{{\AA}}       
\catcode`\'=\active\def'{\c C}        
\catcode`\¯=\active\def¯{{\O}}        
\catcode`\¸=\active\def¸{\Pi}         
\catcode`\Î=\active\defÎ{{\OE}}       
\catcode`\®=\active\def®{{\AE}}       
\catcode`\×=\active\def×{\diamond}    
\catcode`\¡=\active\def¡{\accent'27}  
\catcode`\Ó=\active\defÓ{''}          
\catcode`\±=\active\def±{\pm}         
\catcode`\È=\active\defÈ{\gg}         
\catcode`\À=\active\defÀ{?`}          
\catcode`\Ð=\active\defÐ{--}          
\catcode`\Ñ=\active\defÑ{---}         

\catcode`\Š=\active\defŠ{\"a}        
\catcode`\'=\active\def'{\"e}        
\catcode`\•=\active\def•{\"{\i}}     
\catcode`\š=\active\defš{\"o}        
\catcode`\Ÿ=\active\defŸ{\"u}        
\catcode`\Ø=\active\defØ{\"y}        
\catcode`\€=\active\def€{\"A}        
\catcode`\…=\active\def…{\"O}        
\catcode`\†=\active\def†{\"U}        
\catcode`\‡=\active\def‡{\'a}        
\catcode`\Ž=\active\defŽ{\'e}        
\catcode`\'=\active\def'{\'{\i}}     
\catcode`\—=\active\def—{\'o}        
\catcode`\œ=\active\defœ{\'u}        
\catcode`\ƒ=\active\defƒ{\'E}        
\catcode`\ˆ=\active\defˆ{\`a}        
\catcode`\=\active\def{\`e}        
\catcode`\"=\active\def"{\`{\i}}     
\catcode`\˜=\active\def˜{\`o}        
\catcode`\=\active\def{\`u}        
\catcode`\Ë=\active\defË{\`A}        
\catcode`\‹=\active\def‹{\~a}        
\catcode`\–=\active\def–{\~n}        
\catcode`\›=\active\def›{\~o}        
\catcode`\Ì=\active\defÌ{\~A}        
\catcode`\"=\active\def"{\~N}        
\catcode`\Í=\active\defÍ{\~O}        
\catcode`\‰=\active\def‰{\^a}        
\catcode`\=\active\def{\^e}        
\catcode`\"=\active\def"{\^{\i}}     
\catcode`\™=\active\def™{\^o}        
\catcode`\ž=\active\defž{\^u}        

\let\optionkeymacros\null

 \def\registered{{\ooalign                                          
   {\hfil\kern+.05em\raise.12ex                                     
 \hbox{\tiny R}\hfil\crcr{\footnotesize\mathhexbox20D}}}}           
 \newcommand{\TrueMath}    [1]{\mbox{$#1$}}                  
 \def\half{\TrueMath{\leavevmode\kern.1em \raise.5ex
                     \hbox{\the\scriptfont0 1}
                     \kern-.1em / \kern-.15em\lower.25ex
                     \hbox{\the\scriptfont0 2}
           }         }
 \newcommand{\onR}[1]{\in\mathbb{R}^{#1}}
 \def\cool#1{{\fontfamily{cmss}\selectfont #1}}
 \def\cools#1{{\emph{\fontfamily{cmss}\selectfont #1}}}
 \def\dots{\TrueMath{\ldots}}
\def\set12{\newfont{\size12}{cmbx12}}

\renewcommand{\frac}[2]{\TrueMath{\TrueMath{#1}\over\TrueMath{#2}}}
\newcommand{\define}   {\TrueMath{\,\buildrel \triangle \over = \,}}
\newcommand{\sign}      {\hbox{\rm sign}}
\def\lie#1#2{\TrueMath{{\cal L}_{\bf #1}(#2)}}               
\def\lies#1#2#3{\TrueMath{{\cal L}_{\bf #1}^{#3}(#2)}}       
\def\simless
{\mathop{\raise0.50ex\hbox{$<$}\kern-0.75em\lower0.50ex\hbox{$\sim$}}\nolimits}
\def\simgreat
{\mathop{\raise0.50ex\hbox{$>$}\kern-0.75em\lower0.50ex\hbox{$\sim$}}\nolimits}
\usepackage{mathtools}
\usepackage{amsmath}
\usepackage{amsthm}
\usepackage{amsfonts}
\usepackage{amsbsy}
\usepackage{amssymb}
\usepackage{latexsym}

\usepackage{wrapfig}
\usepackage{natbib,url}

\usepackage{multirow}
\usepackage{fullpage}
\usepackage{subfigure}
\usepackage{xfrac}
\usepackage{mathrsfs}
\usepackage{cases}
\usepackage{empheq}

\usepackage{algorithmic}
\usepackage{algorithm}
\usepackage{subeqnarray}

\usepackage{soul}
\usepackage{xcolor}
\usepackage{comment}
\usepackage{minitoc}

\usepackage[font=scriptsize]{caption}

\newcommand*\reddot{{\protect \includegraphics[width=0.8em]{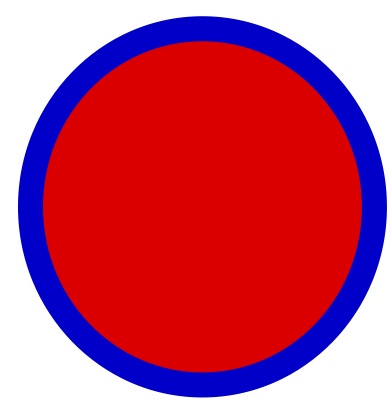}}}
\newcommand*\greendot{{\protect \includegraphics[width=0.8em]{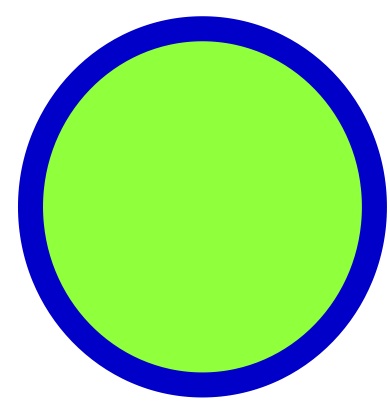}}}

\usepackage{color}

\newlength\myindent
\newtheorem*{theorem*}{Theorem}
\newtheorem*{proposition*}{Proposition}
\newtheorem{proposition}{Proposition}

\newtheorem{definition}{Definition}

\newtheorem*{example*}{Example}
\newtheorem*{property*}{Property}

 \usepackage{tikz}

 \usepackage{marvosym}
 \usepackage{amssymb}
 \usepackage{pifont}
 \usepackage{xspace}

 \def\*tr{^{\,*T}}

 \renewcommand{\onR}[1]{\in\mathbb{R}^{#1}}

 \DeclareSymbolFont{rsfs}{U}{rsfs}{m}{n}
 \DeclareSymbolFontAlphabet{\mathrsfs}{rsfs}
 \def\Cont#1{\mathrsfs{C}^{#1}}

 \def\cool#1{{\fontfamily{cmss}\selectfont #1}}
 \def\cools#1{{\emph{\fontfamily{cmss}\selectfont #1}}}

 \newcounter{HCRLalgorno}

\newcommand\moveFigureUp[2]{%
   \raisebox{#1}{#2}}

\makeatletter
 \newcommand{\customlabel}[2]{%
   \protected@write \@auxout {}{\string \newlabel {#1}{{#2}{}}}}
\makeatother

\journal{The International Journal of Robotics Research}
\volume{XX}
\issue{XX}
\copyrightline{$\copyright$ Copyright}
\firstpage{1}
\lastpage{39}
\doi{doi number}
\articletype{Article Type}
\pubyear{2015}

\begin{document}


\title{A Framework for Planning and Controlling Non-Periodic Bipedal Locomotion}


\author{Ye Zhao$^1$, Benito R. Fernandez$^2$, and Luis Sentis$^{1}$\thanks{Corresponding author;
                 E-mail addresses: yezhao@utexas.edu, 
                                 benito@austin.utexas.edu, 
                                 lsentis@austin.utexas.edu.}
}

\address{$^1$Human Centered Robotics Laboratory, The University of Texas at Austin, TX, USA.\\
$^2$Neuro-Engineering Research and Development Laboratory, The University of Texas at Austin, TX, USA.
}


\maketitle

\begin{abstract}
This study presents a theoretical framework for planning and controlling agile bipedal locomotion based on robustly tracking a set of non-periodic apex states. Based on the prismatic inverted pendulum model, we formulate a hybrid phase-space planning and control framework which includes the following key components: (1) a step transition solver that enables dynamically tracking non-periodic apex or keyframe states over various types of terrains, (2) a robust hybrid automaton to effectively formulate planning and control algorithms, (3) a phase-space metric to measure distance to the planned locomotion manifolds, and (4) a hybrid control method based on the previous distance metric to produce robust dynamic locomotion under external disturbances. Compared to other locomotion frameworks, we have a larger focus on non-periodic gait generation and robustness metrics to deal with disturbances. Such focus enables the proposed control framework to robustly track non-periodic apex states over various challenging terrains and under external disturbances as illustrated through several simulations. Additionally, it allows a bipedal robot to perform non-periodic bouncing maneuvers over disjointed terrains.
\end{abstract}


\section{Introduction}
\label{sec:intro}

Humanoid and legged robots may soon nimbly maneuver over highly rough terrains and cluttered environments. This paper formulates a new framework for the generation of trajectories and an optimal controller to achieve locomotion in those types of environments using a phase-space formalism.  Using prismatic inverted pendulum dynamics and given a set of desired apex states, we present a phase-space planner that can precisely negotiate the challenging terrains.  The resulting trajectories are formulated as phase-space manifolds.  Borrowing from sliding mode control theory, we use the newly defined manifolds and a Riemannian distance metric to measure deviations due to external disturbances or model uncertainties.  A control strategy based on dynamic programming is proposed that steers the locomotion process towards the planned trajectories. Finally, we devise a robust hybrid automaton to effectively formulate control algorithms that involve both continuous and discrete input processes for advanced disturbance recovery. 

[\cite{raibert1986legged}] pioneered robust hopping locomotion of point-foot monoped and bipedal robots using simple dynamical models but with limited applicability to semi-periodic hopping motions. His focus is on dynamically stabilizing legged robots. Instead, our focus is on precisely tracking apex states, i.e. a discrete set of desired robot center of mass (CoM) positions and velocities along the locomotion paths. Such capability is geared towards the design of highly non-periodic gaits in cluttered environments or the characterization of dynamic gait structure in a generic sense. [\cite{pratt2001virtual}] achieved point foot biped walking using a virtual model control method but with limited applicability to mechanically supported robots. Unsupported point foot biped locomotion in moderately rough terrains has been recently achieved by [\cite{grizzle2014models}] and [\cite{ramezani2014performance}] using Poincar\'e stability methods. However, Poincar\'e maps cannot be leveraged to achieving non-periodic gaits for highly irregular or disjointed terrains, which is one of the main applications of our proposed framework. 

In [\cite{frazzoli2001robust}] a robust hybrid automaton is introduced to achieve time-optimal motion planning of a helicopter in an environment with obstacles. The same group studies robustness to model uncertainties [\cite{schouwenaars2003robust}] but ignores external disturbances. More recently, [\cite{majumdar2013robust}] accounts for external disturbances (like cross-wind) by computing funnels via Lyapunov functions and switching between these funnels for maneuvering unmanned air vehicles in the presence of obstacles and disturbances. We apply some of these concepts to point-foot locomotion. Our dynamic system is hybrid, i.e., possessing a different set of dynamic equations for each contact stage.  As a result, we propose a hybrid control algorithm that switches states when the physical system changes the number of contacts.  We use the hybrid automaton framework as a tool for planning and control of bipedal locomotion.  We in fact extend their use of hybrid automaton to accommodate for hybrid systems. Additionally we re-generate phase-space trajectories on demand while the previous works rely on pre-generated primitives. 

Optimal control of legged locomotion over rough terrains are explored in [\cite{Byl:09, kuindersma2015optimization, dai2012optimizing, siyuan2015optimization}]. [\cite{manchester2011stable}] proposed a control technique to stabilize non-periodic motions of under-actuated robots with a focus on walking over uneven terrain. The control is achieved by constructing a lower-dimensional system of coordinates transverse to the target cycle and then computing a receding-horizon feedback controller to exponentially stabilize the linearized dynamics. However, this work does not generate motion trajectories, which is a central part of our work. Also, in contrast with these works, we propose a metric for robustness to recover from disturbances. In [\cite{saglam2014robust}], a controller switching strategy for walking on irregular terrains is proposed. They optimize policies for switching between a set of known controllers. Their method is further extended to incorporate noise on the terrain and through a value iteration process they achieve a certain degree of robustness through switching. However, they focus on mesh learning by considering a set of known controllers. Instead, our paper is focused on creating new controllers from scratch for general types of terrains. Additionally, their work is focused on 2D locomotion whereas we focus on 3D.

A general optimal motion control framework for behavior synthesis of human-like avatars is presented in [\cite{mordatch2012discovery}]. One key missing aspect is quantifying robustness and analyzing feedback stability. Additionally, this work does not address locomotion of point foot robots.

In light of this discussion, our contributions are the following: (1) we formulate a hybrid automaton to characterize non-periodic locomotion dynamics, (2) using the automaton, we synthesize motion plans in the phase-space to maneuver over irregular terrains while tracking a set of desired keyframes (i.e. apex states), (3) a phase-space manifold is created with a Riemannian distrance metric to measure nominal trajectory deviations and design an in-step controller, and (4) we derive an optimal control framework to recover from disturbances and uncertainty and study its stability. Overall, the key difference with previous works is our focus on trajectory generation and robust control of non-periodic gaits. We are less centered on dynamic balance or moving from a start to a finish location but instead on tracking desired keyframes composed of discrete robot CoM positions and velocities. As such our framework provides a greater level of granularity that makes it more suitable for designing gates in cluttered environments.

\section{Related Work} 
\label{sec:relatedwork}

The Capture Point method [\cite{pratt2006capture}] provides one of the most practical frameworks for locomotion. Sharing similar core ideas, divergent component of motion \mbox{[\cite{takenaka2009real}]} and extrapolated center of mass \mbox{[\cite{hof2008extrapolated}]} were independently proposed. Extensions to the Capture Point method, [\cite{morisawa2012balance, englsberger2015three}], allow locomotion over rough terrains. Motion planning techniques based on interpolation through kinematic configurations have been explored, among other works, by [\cite{hauser2014fast}] and [\cite{pham2013kinodynamic}]. Those techniques are making great progress towards mobility and locomotion in various kinds of environments. The main difference from these studies is that our controller explicitly accounts for robustness and stability to achieve under-actuated dynamic walking. If these works were to be implemented in unstable robots such as point-foot bipeds, they would lose balance and fail to recover. Our planner allows for continuous recovery without explicitly controlling the robot's center of mass.

In [\cite{hobbelen2007disturbance}], a gait sensitivity norm is presented to measure disturbance rejection during dynamic walking. In [\cite{hamed2015robust}], sensitivity analysis with respect to ground height variations is performed to model robustness of orbits. These techniques are limited to cyclic walking gaits. The work in [\cite{arslan2012reactive}] unifies planning and control to provide robustness. In [\cite{nguyen2015optimal}], an optimal robust controller is designed based on control Lyapunov function to achieve bipedal locomotion with model uncertainties. However, the techniques in these two works are only applied to planar robots.

Numerous studies have focused on recovery strategies upon disturbances [\cite{hofmann2006robust, zhao2013biologically}]. Various recovery methods have been proposed based on ankle, hip, knee, and stepping strategies [\cite{kuo1992human, stephens2007humanoid}]. In [\cite{hyon2007disturbance}], a stepping controller based on ground contact forces is implemented in a humanoid robot. The study in [\cite{komura2005feedback}] controlship angular momentum to achieve planar bipedal locomotion. In our study, we simultaneously control angular momentum, CoM height and foot placements to achieve unsupported rough terrain walking. 

Online trajectory optimization is explored in [\cite{tassa2012synthesis, audren2014model}] for complex humanoid behaviors. [\cite{stephens2010push}] uses model predictive control (MPC) for push recovery by planning future steps. [\cite{wieber2006trajectory}] presents a linear MPC scheme for zero moment point control with perturbations. The problem of MPC is that it cannot find a globally optimal solution due to the finite horizon. In contrast, we use dynamic programming to exhaustively search for the global optimal solution. Since the inverted pendulum model has low dimensional states, the ``curse of dimensionality'' is not an issue for our case. 

\section{Problem Definition}
\label{sec:ProblemDefinition}

We first present basic control formalism and manifold analysis that will allow us to characterize, plan and control non-periodic locomotion processes in later sections.

\subsection{System Equations}\label{sec:SystemEquations}
Legged robots can be characterized as Multi-Input/Multi-Output (MIMO) systems.  Let us assume that a bipedal robot can be characterized by $n_j$ joint degrees of freedom (DOF), $\boldsymbol{q}=[q_1,q_2,\dots,q_{n_j}]^T\onR{n_j}$. Letting $\boldsymbol{x}(t)=[\boldsymbol{q}^T(t),\dot{\boldsymbol{q}}^T(t))]^T\in\mathbb{R}^{n}$, be the state-space vector ($n=2n_j$), $\boldsymbol{u}(t)\in\mathbb{R}^{m}$, represents the control input vector (generalized torques and forces), and defining $\boldsymbol{f}(\boldsymbol{x}(t))$, $\boldsymbol{g}(\boldsymbol{x}(t))$, and $\boldsymbol{h}(\boldsymbol{x}(t))$ in the obvious manner, the mechanical model is expressed in state variable form as
\begin{subeqnarray}\label{eqs:plant} 
\dot{\boldsymbol{x}}(t) &=&  \boldsymbol{f} (\boldsymbol{x}(t)) 
                    +   \boldsymbol{g} (\boldsymbol{x}(t)) \boldsymbol{u}(t)
                    +   \boldsymbol{J}_d (\boldsymbol{x}(t)) \boldsymbol{d}(t),  \label{eq:plantStates}\\
   \boldsymbol{y}(t)   &=&  \boldsymbol{h} (\boldsymbol{x}(t)),                  \label{eq:plantOutputs}
\end{subeqnarray}
where, $\boldsymbol{d}(t)$ represent the generalized external disturbance forces, and $\boldsymbol{J}_d(\boldsymbol{q}(t))$ is the disturbance distribution matrix.  
The output vector, $\boldsymbol{y}(t) = [y_1,y_2,\dots,y_p]^T\in \mathbb{R}^{p}$ is generated by $\boldsymbol{h}(\boldsymbol{x}(t))$, that may represent positions and/or velocities in the task space. Without loss of generality, let us consider systems in the normal form, where $\boldsymbol{h}(\cdot)$ is at least $\Cont{r}$, where $r$ is the relative order of the output. The disturbances and modeling errors satisfy the matching conditions [\cite{Fernandez:PhD}]. 
\subsection{System Normalization for Phase-Space Control Design}
\label{sec:Normalization} 
General robotic systems are not in normal form, but we can transform them by finding what relative order of the output derivatives are explicitly controllable .  Each of the outputs $y_i$ in Eq.~(\ref{eq:plantOutputs}) has a relative order $r_i$, defined by the smallest derivative order where the control appears,
 \begin{subeqnarray}
     y_i^{[k]}   &=& \dfrac{d^ky_i}{dt^k} = \lies{\boldsymbol{f}}{h_i(\boldsymbol{x})}{k} 
                +  \lie{\boldsymbol{g}}{\lies{\boldsymbol{f}}{h_i(\boldsymbol{x})}{k-1}}\boldsymbol{u},\\
     \lie{\boldsymbol{g}}{\lies{\boldsymbol{f}}{h_i(\boldsymbol{x})}{k-1}}   &  =  & 0 \quad ~{\rm for} \quad 0\le k<r_i
     \label{eq:RO},\\
     y_i^{[r_i]} &=& \lies{\boldsymbol{f}}{h_i(\boldsymbol{x})}{r_i} 
                +  \lie{\boldsymbol{g}}{\lies{\boldsymbol{f}}{h_i(\boldsymbol{x})}{r_i-1}}\boldsymbol{u},\\
     \lie{\boldsymbol{g}}{\lies{\boldsymbol{f}}{h_i(\boldsymbol{x})}{r_i-1}} & \ne & 0 \quad ~
          \quad \forall \boldsymbol{x} \in \mathbb{S}_i \subset \mathbb{R}^n,
     \label{eq:RO_2}
  \end{subeqnarray}
where, $\lies{\boldsymbol{f}}{h_i(\boldsymbol{x})}{0}=h_i(\boldsymbol{x})$, $\lie{\boldsymbol{f}}{\boldsymbol{h}}$ and $\lie{\boldsymbol{g}}{\boldsymbol{h}}$ are the directional Lie derivatives of function $\boldsymbol{h}(\boldsymbol{x})$ in the directions of $\boldsymbol{f}(\boldsymbol{x})$ and $\boldsymbol{g}(\boldsymbol{x})$ respectively [\cite{Isidori:Book}], and $\mathbb{S}_i$ is the output-controllable subspace, where the Lie derivative in Eq.~(\ref{eq:RO_2}) does not vanish,
\begin{equation}
\mathbb{S}_i~=~\Bigl\{~\boldsymbol{x} \onR{n}~~\left|~~\lie{\boldsymbol{g}}{\lies{\boldsymbol{f}}{h_i(\boldsymbol{x})}{r_i-1}}~\ne~0\right.\Bigr\}.
\label{uncontrollableSpace}
\end{equation}
The relative order tells us that the $r_i^{\rm th}$-derivative of output $y_i$ can be explicitly controlled. The region where $\mathbb{S}_i$ vanishes, entails the system looses relative order and hence the $r_i^{\rm th}$-derivative is no longer controllable (at least explicitly). For a controllable system, $r_i \le n$. Following the normalization  procedure, we get the output controllable subspace,
\begin{subeqnarray}\label{eq:xi}
  \xi_{i,1} &=& y_i = h_i(\boldsymbol{x}) = \lies{\boldsymbol{f}}{h_i(\boldsymbol{x})}{0}, \\[-2.5mm]
              &\dots& \nonumber\\[-2.5mm]
  \xi_{i,j} &=& y_i^{[j-1]} = \dot {\xi}_{i,j-1} = \lies{\boldsymbol{f}}{h_i(\boldsymbol{x})}{j-1} 
                              \quad ~{\rm for} \quad 1<j<r_i,\\[-2.5mm]
              &\dots& \nonumber\\[-2.5mm]
            &~& y_i^{[r_i]}   = \dot {\xi}_{i,r_i}   = \lies{\boldsymbol{f}}{h_i(\boldsymbol{x})}{r_i} 
                              + \lie{\boldsymbol{g}}{\lies{\boldsymbol{f}}{h_i(\boldsymbol{x})}{r_i-1}}\boldsymbol{u}(t).
\end{subeqnarray}
The output space variables, $\boldsymbol{\xi}_i = [\xi_{i,1}, \xi_{i,2}, \dots, \xi_{i,r_i-1}]^T\onR{r_i}$ represent the phase-space for the $i$-th output. For instance, the output phase-space for locomotion control could be chosen to be the robot's center of mass. We can concatenate all $\boldsymbol{\xi}_i$, $\forall i=1,2,\dots,m$ into a single phase-space vector $\boldsymbol{\xi} = [\boldsymbol{\xi}_1^T, \boldsymbol{\xi}_2^T, \dots, \boldsymbol{\xi}_m^T]^T\onR{r}$, where, $r=\sum{r_i}$.  
For phase-space motion, we define a phase-space manifold $\mathcal{M}_i$ for each task-space output, $y_i$, in terms of its phase-space vector, $\boldsymbol{\xi}_i$,
\begin{equation}
 \mathcal{M}_i~=~\Bigl\{~\boldsymbol{\xi}_i \in \mathbb{R}^{r_i}\subset\mathbb{R}^n~~\left|~~\sigma_i
             ~\define~\sigma_i(\boldsymbol{\xi}_i)~=~ 0\right.\Bigr\}, \label{eq:surface}
\end{equation}
where, $\sigma_i$ is referred to as the $i^{\rm th}$ element of \cools{deviation vector}, which measures the deviation distance from the manifold $\mathcal{M}_i$ using a Riemannian metric.  In order to be able to control this deviation, the order of the manifold is one less than the relative order of the $i^{\rm th}$-output, i.e., $r_i-1$. For most legged robots (not considering actuator dynamics), the relative order is $r=2$.

\section{Prismatic Inverted Pendulum Dynamics on a Parametric Surface}
\label{sec:model}

\begin{figure*}
 \centering
 \includegraphics[width=0.95\linewidth]{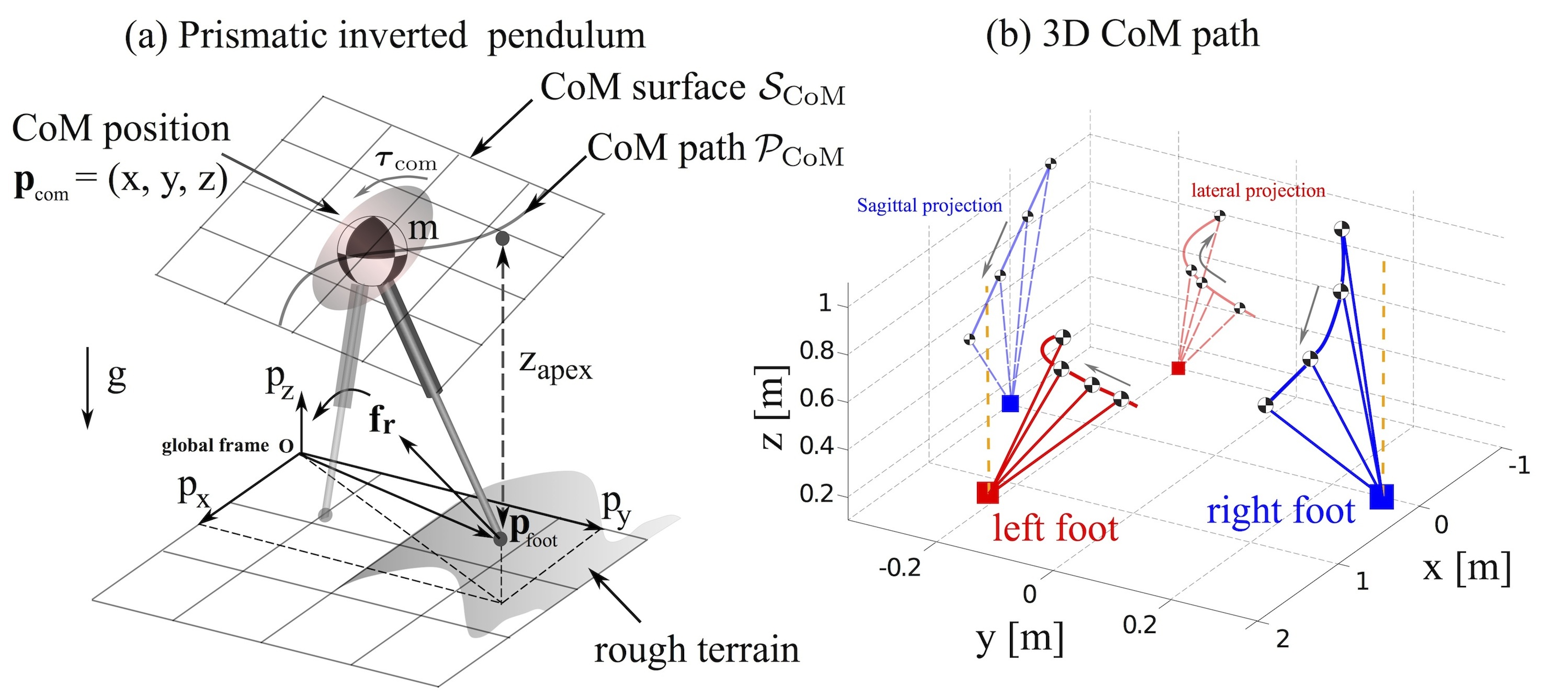}
 \caption{\captionsize 3D prismatic inverted pendulum model. (a) We define a prismatic inverted pendulum model with all of its mass located at its base while equipping it with a flywheel to generate moments. We restrict the movement of the center of mass to 3D planes (surfaces) $\mathcal{S}_{\rm CoM}$. (b) shows motions of pendulum dynamics restricted to a 3D plane.}
 \label{fig:Inverted_Pendulum}
\end{figure*}
The dynamics of point foot bipedal robots in generic terrain topologies during single contact can be mechanically approximated as an inverted pendulum model (see Fig.~\ref{fig:Inverted_Pendulum}).  Our model consists of a prismatic massless joint with all the mass concentrated on the hip position, defined as the 3D CoM position, $\boldsymbol{p}_{\rm com}=(x, y, z)^T$, and a flywheel spinning around it, with orientation angles $\boldsymbol{R} = (\phi, \theta, \psi)^T$.  The objective of locomotion is to move the robot's CoM along a certain path from point A to B over a terrain. As such, we first specify a 3D surface, $\mathcal{S}_{\rm CoM}$, where the CoM path will exist, which in general, may have the following implicit form,
\begin{align}\label{eq:genericSurface}
 \mathcal{S}_{\rm CoM}~=~\Bigl\{\boldsymbol{p}_{\rm com} \onR{3}
            ~~\left|~~\psi_{\rm CoM}(\boldsymbol{p}_{\rm com})~=~ 0\right.\Bigr\}.
\end{align}
This surface can be specified in various ways, such as via piecewise arc geometries [\cite{mordatch2010robust, srinivasan2006computer}]. Once the controller is designed, the CoM will follow a concrete trajectory $\mathcal{P}_{\rm CoM}$ (as shown in Fig.~\ref{fig:Inverted_Pendulum}), which we specify via piecewise splines described by a progression variable, $\zeta \in [\zeta_{j-1}, \zeta_j]$, for the $j^{\rm th}$ path manifold, i.e.
\begin{eqnarray}\label{eq:splineManifold}
 \mathcal{P}_{\rm CoM} &= \bigcup_{j}^{} \mathcal{P}_{{\rm CoM}_j} 
   \subseteq \mathcal{S}_{\rm CoM},\hspace{0.2in} 
   \mathcal{P}_{{\rm CoM}_j} &= 
   \Bigl\{~\boldsymbol{p}_{{\rm com}_j} \in \mathbb{R}^3
         ~~\left|~~\boldsymbol{p}_{{\rm com}_j} = \sum_{k = 0}^{n_p} \boldsymbol{a}_{jk} \zeta^k\right.\Bigr\},
\end{eqnarray}
where $n_p$ is the order of the spline degree.  The progression variable $\zeta$ is therefore the arc length along the CoM path acting as the Riemannian metric for distance. Each $\boldsymbol{a}_{jk} \in \mathbb{R}^3$ is the coefficient vector of $k^{\rm th}$ order. To guarantee the spline smoothness, $\boldsymbol{p}_{\rm com}$ requires the connection points, i.e. the knots at progression instant $\zeta_j$, to be $\Cont{n_p-1}$ continuous, 
\begin{align}
\boldsymbol{p}^{[l]}_{{\rm com}_j}(\zeta_j) 
  = \dfrac{d^l\boldsymbol{p}_{{\rm com}_j}}{d\zeta^l}(\zeta_j) 
  = \boldsymbol{p}^{[l]}_{{\rm com}_{j+1}}(\zeta_j), 
  \quad \forall \; 0 \leq l \leq n_p-1.
\end{align}
The purpose of introducing the CoM manifold $\mathcal{S}_{\rm CoM}$ is to constrain CoM motions on surfaces that are designed to conform to generic terrains while allowing free motion within this surface. Following a concrete CoM path is achieved by selecting proper control inputs as we will see further down. The CoM path manifold, $\mathcal{P}_{\rm CoM}$ (embedded in $\mathcal{S}_{\rm CoM}$), can be represented in the phase-space, $\boldsymbol{\xi}$.  We call this representation as the \emph{phase-space manifold} and define it as,
\begin{align}\label{eq:Mcom}
 \mathcal{M}_{\rm CoM} = \bigcup_{j}^{} \mathcal{M}_{{\rm CoM}_j}, 
   \qquad 
    \mathcal{M}_{{\rm CoM}_j} = \Bigl\{~\boldsymbol{\xi} \in \mathbb{R}^6
            ~~\left|~~\sigma_j(\boldsymbol{\xi})~=~ 0\right.\Bigr\},
\end{align}
which is the core manifold used in our phase-space planning and control framework.  The function $\sigma_j(\boldsymbol{\xi})$ is an implicit function in the phase space measuring the distance to the manifold.

\subsection{Dynamic Equations of Motion}
Besides the CoM path surface previously described, pendulum dynamics can be characterize via formulating the dynamic balance of moments of the pendulum system. For our single contact scenario, the sum of moments, ${\bf m}_i$, with respect to the global reference frame (see Fig.~\ref{fig:Inverted_Pendulum}) is 
\begin{equation}\label{eq:balance-2}
\sum_{i} {\bf m}_i = -\boldsymbol{p}_{\rm foot} \times \boldsymbol{f}_{r} + \boldsymbol{p}_{\rm com} \times \Big( \boldsymbol{f}_{\rm com} + m \, \boldsymbol{g} \Big) +
\boldsymbol{\tau}_{\rm com} = 0,
\end{equation}
where, $\boldsymbol{p}_{\rm foot} = (x_{\rm foot}, y_{\rm foot}, z_{\rm foot})^T$ is the position of the foot contact point (i.e., left or right foot), $\boldsymbol{f}_{r}$ is the three dimensional vector of ground reaction forces, $\boldsymbol{f}_{\rm com} = m (\ddot{x}, \ddot{y}, \ddot{z})^T$ is the vector of center of mass inertial forces, $\boldsymbol{\tau}_{\rm com} = (\tau_x, \tau_y, \tau_z)^T$ is the vector of angular moments of the modeled flywheel attached to the inverted pendulum, $m$ is the total mass, and $\boldsymbol{g} \in \mathbb{R}^3$ corresponds to the gravity field.  The system's linear force equilibrium can be formulated as $\boldsymbol{f}_{r} = \boldsymbol{f}_{\rm com} + m \, \boldsymbol{g}$, allowing us to simplified Eq.~(\ref{eq:balance-2}) to:
\begin{equation}\label{eq:balance-3}
\Big(\boldsymbol{p}_{\rm com} - \boldsymbol{p}_{\rm foot} \Big) \times (\boldsymbol{f}_{\rm com} + m \, \boldsymbol{g}) = -\boldsymbol{\tau}_{\rm com}.
\end{equation}
For our purposes, we consider only the class of prismatic inverted pendulums whose center of mass is restricted to a path surface $\mathcal{S}_{\rm CoM}$ as indicated in Eq.~(\ref{eq:genericSurface}). Moreover, for simplicity we only consider 3D surfaces that are invariant to the lateral CoM coordinate as we will see in Eq. (\ref{eq:linear2DSurface}). This consideration allows to decouple the dynamics of the Sagittal plane from those of the lateral plane. Considering as our output state the CoM positions, $\boldsymbol{p}_{\rm com}$, the state space, $\boldsymbol{\xi}= (\boldsymbol{p}_{\rm com}^T, \boldsymbol{\dot{p}}_{\rm com}^T)^T = (x,y,z,\dot{x},\dot{y},\dot{z})^T \in \Xi \subseteq \mathbb{R}^6$ is the phase-space vector. Then from Eq. (\ref{eq:balance-3}) it can be shown that the prismatic inverted pendulum model for a $q^{\rm th}$ walking step, is simplified to the following control system
\begin{align}\label{eq:accel}
\dot{\boldsymbol{\xi}} = \boldsymbol{\mathcal{F}}(q, \boldsymbol{\xi}, \boldsymbol{u}) = 
 \begin{pmatrix}
  \dot{x}\\
  \dot{y}\\
  \dot{z}\\
  \omega_q^2 (x - x_{{\rm foot}_q}) - \frac{ \omega_q^2}{mg}\tau_y\\[2mm]
  \omega_q^2 (y - y_{{\rm foot}_q}) - \frac{ \omega_q^2}{mg}\tau_x\\[2mm]
 a_q\omega_q^2 (z - z_{{\rm foot}_q}) - \frac{a_q\omega_q^2}{mg}\tau_y
 \end{pmatrix},
\end{align}
where the phase-space asymptotic slope is defined as
\begin{align}\label{eq:CoMaccelRate}
\omega_q = \sqrt{\dfrac{g}{z_{{\rm apex}_q}}}, \;{\rm with}\; z_{{\rm apex}_q} = (a_q \cdot x_{{\rm foot}_q} + b_q - z_{{\rm foot}_q}),
\end{align}
where $g$ is the gravity constant, $a_q$ and $b_q$ are the slope and constant scalars for the linear CoM path surfaces that we consider (see Fig. \ref{fig:3Dsinglecontact}), i.e.
\begin{align}\label{eq:linear2DSurface}
\mathcal{S}_{{\rm CoM}_q} = \left\{(x,y,z) \onR{3} \quad\Big|\quad 
                            \psi_{{\rm CoM}_q}(x,y,z) = z - a_q x - b_q = 0\right\}.
\end{align}
\begin{figure}[t]
 \centering
 \includegraphics[width=0.85\linewidth]{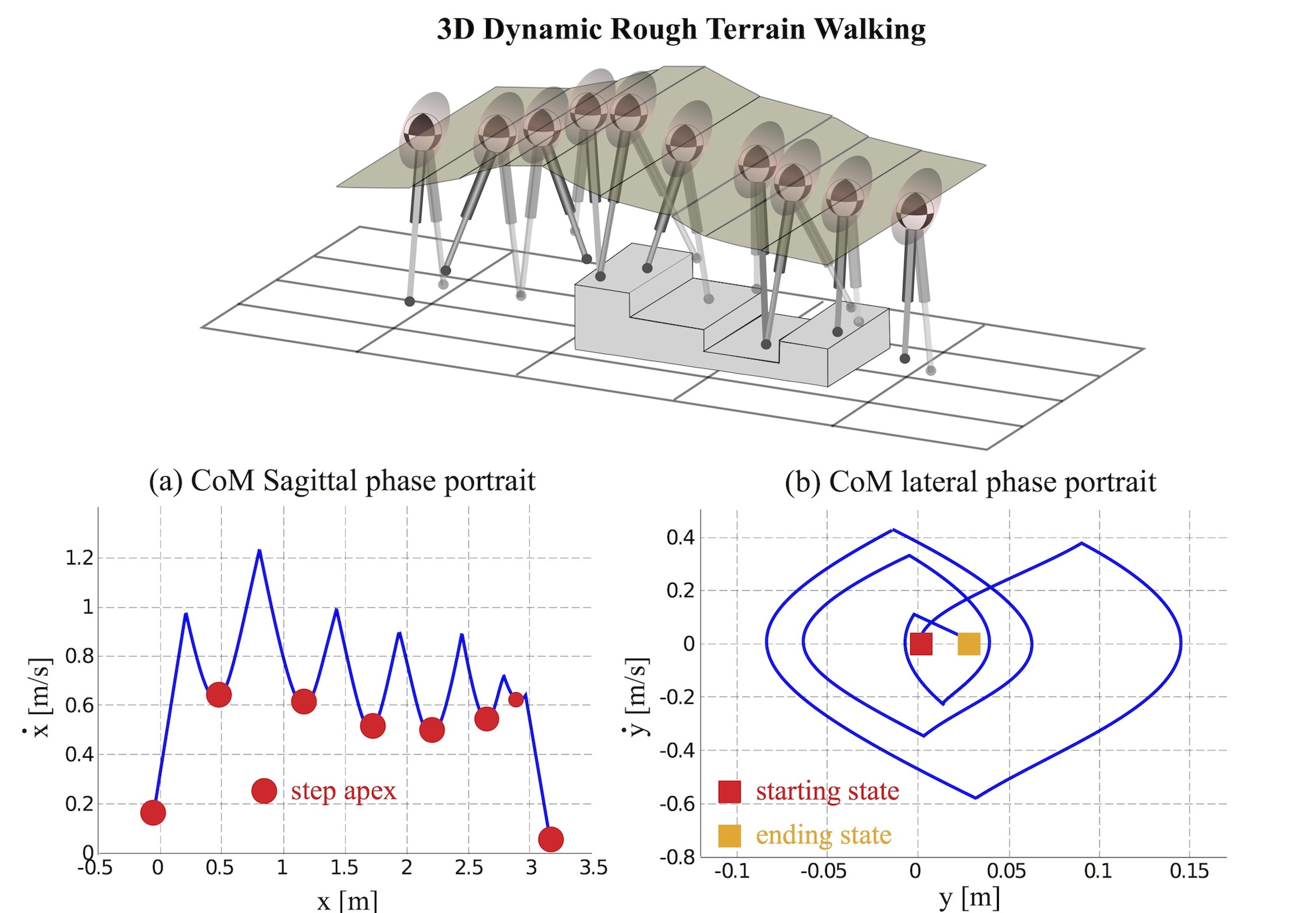}
 \caption{\captionsize 3D phase-space planning. Given step apex conditions, single contact dynamics generate the valley profiles shown in (a). (b) depicts a similar strategy in the lateral plane. However, since foot transitions have already been determined, what is left is to determine foot lateral positions. This is done so the lateral CoM behavior shown in (b) follows a semi-periodic trajectory that is bounded within a closed region.}
 \label{fig:3Dsinglecontact}
\end{figure}
\noindent We have defined $z_{{\rm apex}_q}$ such that it corresponds to the vertical distance between the CoM and the location of the foot contact at the instant where their Sagittal projections are identical, i.e. the \emph{Sagittal Apex}; $\boldsymbol{\mathcal{F}}$ represents a vector field of inverted pendulum dynamics. In general, there will be a hybrid control policy, $\boldsymbol{u}=\boldsymbol{\pi}(q,\boldsymbol{\xi})$, defined by the control variables $\boldsymbol{u} = \{\omega_q, \boldsymbol{\tau}_{{\rm com}_q}, \boldsymbol{p}_{{\rm foot}_q}\} \in \mathcal{U}$,  where, $\mathcal{U}$ is an open set of admissible control values. The sets $\Xi$ and $\mathcal{U}$ are assumed to be compact. 
\begin{definition}[Sagittal Apex]\label{def:sagapex}
The Sagittal apex occurs when the projection of the CoM is equal to the location of the foot contact in the system's Sagittal axis. 
\end{definition}

\subsection{Nominal Phase-Space Trajectory Generation}
We will first focus on the generation of trajectories in the Sagittal plane of the robot's walking reference. Sagittal dynamics are represented \-- ignoring for simplicity, the discrete variable, $q$, \-- in the first and fourth row of the system of Eq. (\ref{eq:accel}), i.e.
\begin{equation}\label{eq:dynx}
\boldsymbol{\dot x} = \boldsymbol{\mathcal{F}_x}(\boldsymbol{x}, \boldsymbol{u_x})=
\begin{pmatrix}
\dot x\\
\omega^2 (x - x_{\rm foot}) - \frac{ \omega^2}{mg}\tau_y
\end{pmatrix}.
\end{equation}
This system would be fully controllable if its continuous control inputs, $\omega$ and $\tau_y$, were unconstrained. However, their limited range urges us to first consider the motion trajectories under nominal (i.e. open loop) values. As we previously motivated in Eq. (\ref{eq:linear2DSurface}), the path manifold, $\mathcal{S}_{\rm CoM}$ is defined a priori to conform to the terrains via methods not considered in this paper. From Eq. (\ref{eq:CoMaccelRate}), once the path manifold is defined and for known contact locations, the set of phase-space asymptotic slopes, $\omega$, is also known as shown in Eq. (\ref{eq:CoMaccelRate}). For simplicity, the nominal flywheel moments are designed to be null, i.e. $\tau_y = 0$. Under these considerations, the following algorithm produces nominal phase-space trajectories of the system's center of mass in the Sagittal direction of reference:\\

\noindent\textbf{Algorithm 1.} Nominal Phase-Space Trajectory Generation.
\begin{itemize}
\item[] \textbf{Input:}
\item[] \textbf{(i)}: $\mathcal{S}_{\rm CoM} \leftarrow \{\mathcal{S}_{{\rm CoM}_q} : [\zeta_{q-1}, \zeta_q] \rightarrow  \mathbb{R}^3, \; \forall q = 1, \ldots, N\}$
\item[] \textbf{(ii)}: $x_{{\rm foot}} \leftarrow \{x_{{\rm foot}_1}, x_{{\rm foot}_2}, \ldots, x_{{\rm foot}_{N}}\}$
\item[] \textbf{(iii)}: $\dot{x}_{{\rm apex}} \leftarrow \{\dot{x}_{{\rm apex}_1}, \dot{x}_{{\rm apex}_2}, \ldots, \dot{x}_{{\rm apex}_N}\}$
\item[] \textbf{(iv)}: $\tau_y(t) \leftarrow {\bf 0}$

\item[] \textbf{Operation:} 
\item[] \textbf{(i)}: $\omega \coloneqq \{\omega_1, \omega_2, \ldots, \omega_N\} \leftarrow \mathcal{S}_{\rm CoM}$
\item[] \textbf{(ii)}: $\ddot x(t) \leftarrow {\rm PIPM} (\omega, \tau_y(t), x_{\rm foot})$
\item[] \textbf{(iii)}: $(x_{k+1}, \dot x_{k+1}) \leftarrow {\rm ODE}(x_k, \dot x_k, \ddot x_k, \epsilon)$

\textbf{Output:}
\item[] Phase-space trajectories $\mathcal{M}_{\rm CoM} := \bigcup_q^{} \mathcal{M}_{{\rm CoM}_q}$
\end{itemize}
A similar algorithm could be written to generate motion in the lateral direction. Here, $\dot x_{\rm apex}$, represents desired apex velocities, {\rm PIPM} represents the prismatic inverted pendulum model on the parametric surfaces outlined in Eq. (\ref{eq:accel}), {\rm ODE} represents the numerical integration of the ordinary differential equation associated with the model, with numerical step $\epsilon$, and $k$ represents the discretization of the output state for numerical integration purposes. Trajectories for multiple steps of a locomotion sequence on rough terrain are simulated using this process in Fig.~\ref{fig:3Dsinglecontact}.

\section{Hybrid Phase-Space Motion Planning}
\label{sec:3d-moplan}

In this section we propose a hybrid robust automaton [\cite{branicky1998unified, frazzoli2001robust, lygeros2008hybrid}] with the following key features: I) an invariant set and a recoverability set to characterize control robustness, i.e., the bundle of attractiveness, and II) a non-periodic step transition strategy based on the previously described phase-space trajectories.  The hybrid automaton governs the planner's behavior across multiple walking steps and as such constitutes the theoretical core of our proposed phase-space locomotion planning framework.

We continue our focus on Sagittal plane dynamics first, then extend the planner to all directions. For practical purposes we will use the symbol $\boldsymbol{x} = \{x, \dot x\}$ to describe the Sagittal state space associated with CoM dynamics. Note that this symbol represents now the output dynamics outlined in Eq. (\ref{eq:xi}) instead of the robot plant of Eq. (\ref{eqs:plant}). Eq. (\ref{eq:Mcom}) can thus be re-considered in the output space as $\mathcal{M}_{{\rm CoM}_q} = \left\{\boldsymbol{x} \in \mathcal{X} ~\Big|~ \sigma_q(\boldsymbol{x}) = 0 \right\}$
 where $\sigma_q$ (the normal distance) representing the deviation from the manifold $\mathcal{M}_{{\rm CoM}_q}$.

\begin{definition}[Invariant Bundle]\label{def:invariantBundle}
A set $\mathcal{B}_q(\epsilon)$ is an invariant bundle if, given $\boldsymbol{x}_{\zeta_0} \in \mathcal{B}_q(\epsilon)$, with $\zeta_0  \in \mathbb{R}_{\geq 0}$, and increment $\epsilon>0$, $\boldsymbol{x}_{\zeta}$ stays within an $\epsilon$-bounded region of $\mathcal{M}_{{\rm CoM}_q}$,
\begin{align}\label{eq:path-manifold}
  \mathcal{B}_q(\epsilon) = \left\{\boldsymbol{x} \in \mathcal{X} 
   \quad\Big|\quad \left|\sigma_q(\boldsymbol{x})\right| \le \epsilon \right\},
 \end{align}
where, $\zeta_0$ and $\zeta$ are initial and current phase progression variables, respectively. $\boldsymbol{x}_{\zeta_0}$ is an initial condition. 
\end{definition}
\noindent This type of bundle characterizes ``robust subspaces'' (i.e., ``tubes'') around nominal phase-space trajectories which guarantee that, if the state initializes within this space, it will remain on it.

\begin{definition}[Finite-Phase Recoverability Bundle]\label{def:recoverBundle}
 The invariant bundle $\mathcal{B}_q(\epsilon)$ around a phase-space manifold $\mathcal{M}_{\rm CoM_q}$ has a finite phase recoverability bundle, $\mathcal{R}_q(\epsilon,\zeta_f)  \subseteq \mathcal{X}$ defined as,
 \begin{align}
   \mathcal{R}_q(\epsilon,\zeta_f) = 
      \left\{ 
        \boldsymbol{x}_{\zeta} \in \mathcal{X}, \quad \zeta_0\le\zeta\le \zeta_f
        \quad\Big|\quad 
        \boldsymbol{x}_{\zeta_f} \in \mathcal{B}_q(\epsilon)
       \right\}.
 \end{align}
\end{definition}

\noindent Note that this bundle assumes the existence of a control policy for recoverability. We will later use these metrics to characterize robustness of our controllers. Visualization of the invariant and recoverability bundles are shown in Fig. \ref{fig:Detailed-2Half-Steps2}.

\begin{figure}
\begin{center}
 \moveFigureUp{5mm}{\includegraphics[width=0.49\textwidth]{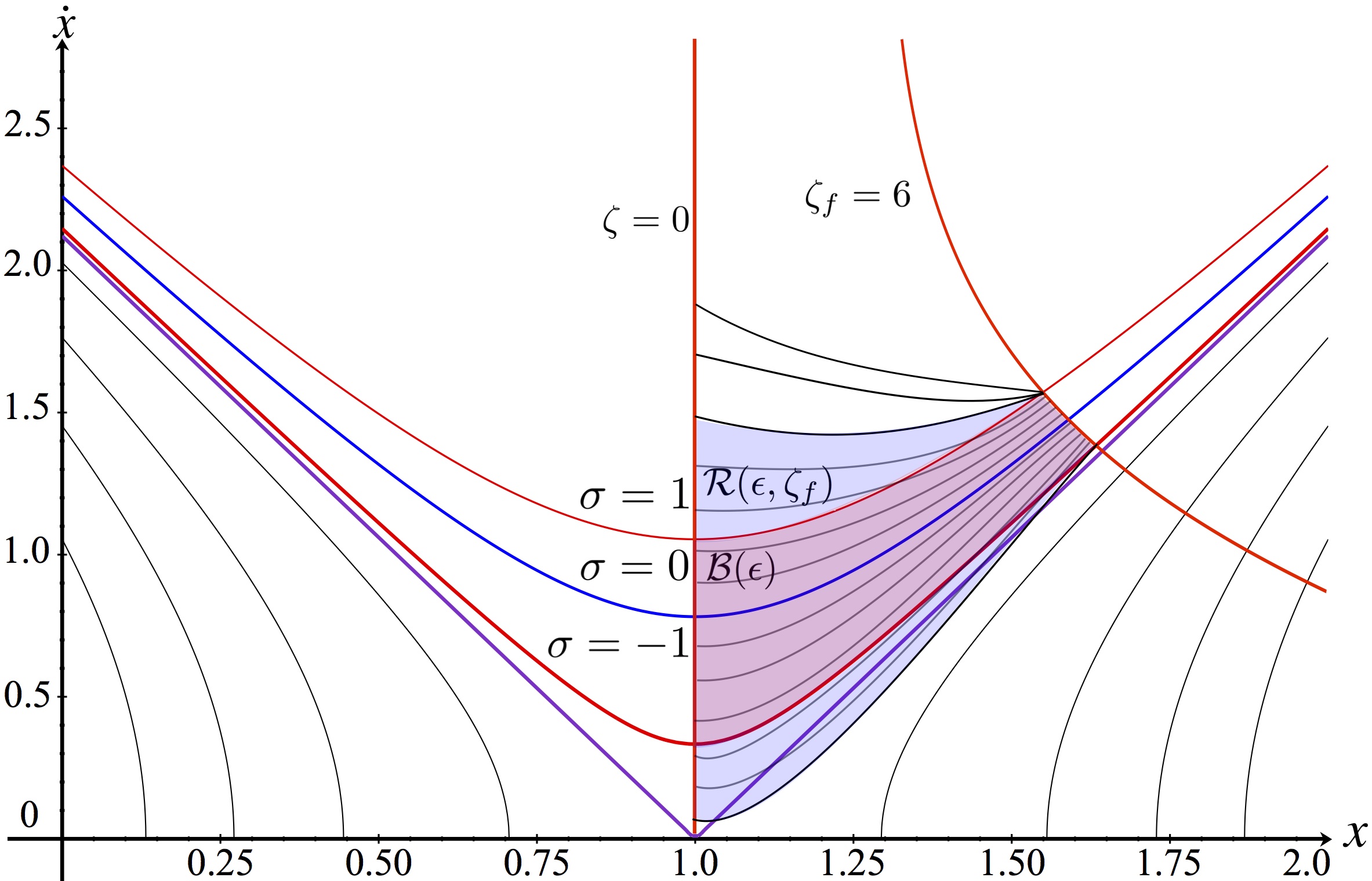}}
 \includegraphics[width=0.48\linewidth]{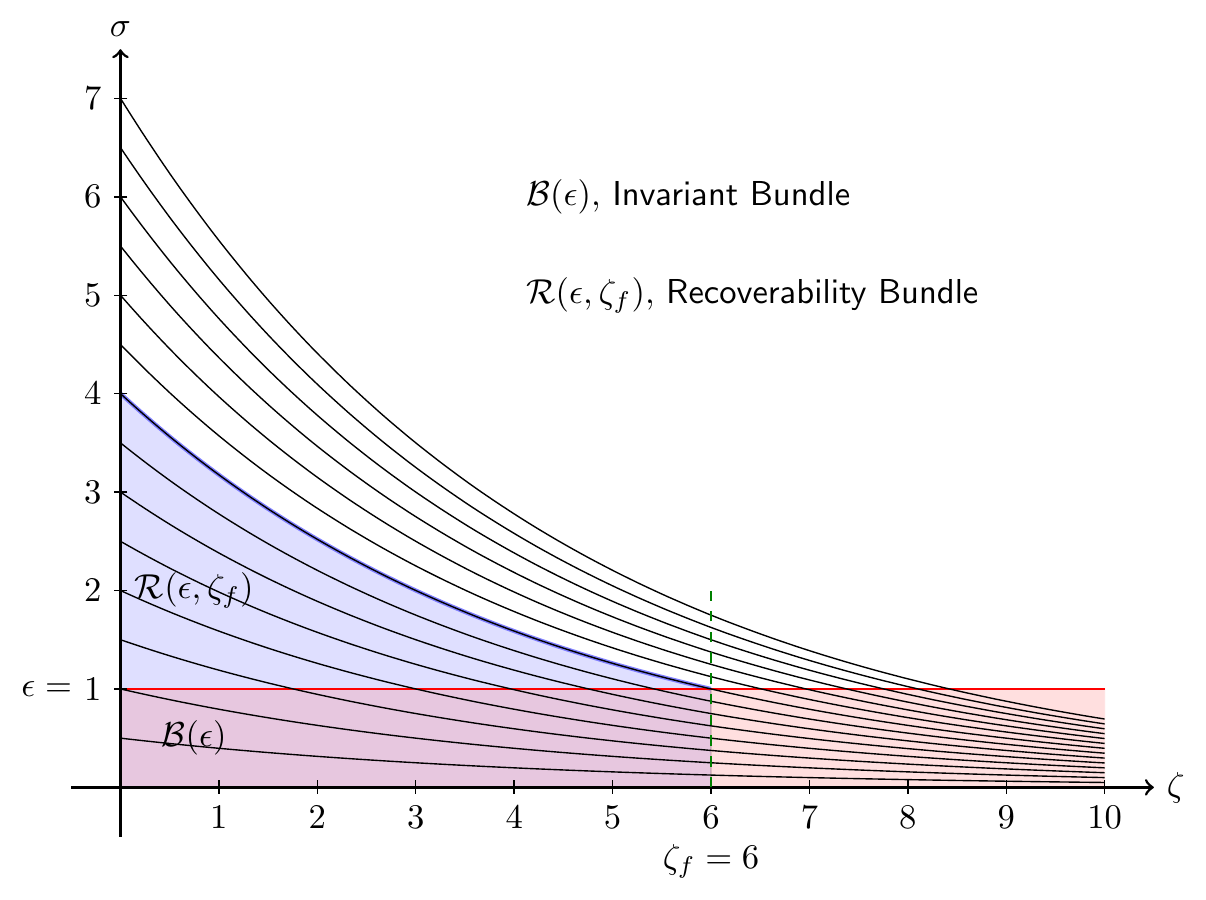}
 \caption{Mapping between Cartesian phase-space and $\zeta-\sigma$ space. The two subfigures show the invariant bundle, $\mathcal{B}(\epsilon)$ (shown in red) and the recoverability bundle, $\mathcal{R}(\epsilon,\zeta_f)$ (shown in blue) in different spaces. The left subfigure shows Cartesian phase-space while the right one shows $\zeta-\sigma$ space ($\sigma$ denotes the phase-space manifold as defined in Eq.~(\ref{eq:Mcom})). The figure on the right only shows positive bundles of $\sigma$ while the negative ones are symmetric about the $\zeta$ axis.  Since this is a Euclidean space, the manifold for a constant $\sigma$ is a horizontal line and constant values of $\zeta$ are vertical lines. If the condition when we expect the transition to occur is at $\zeta=\zeta_f$, the recoverability bundle shows the range of perturbations that can be tolerated at different $\zeta$ -- the system recovers to the invariant bundle before $\zeta_f$.}
 \label{fig:Detailed-2Half-Steps2}
\end{center}
\end{figure}

\subsection{Hybrid Locomotion Automaton}
\label{subsec:RHLA}
Legged locomotion is a naturally hybrid control system, with both continuous and discrete dynamics. We define discrete states $\mathcal{Q} = \{q_{l}, q_{r}, q_{s}\}$ representing the support of the left foot $q_{l}$ or the right foot $q_{r}$ or both $q_{s}$ (stance) feet as shown in Figs.~\ref{fig:3Dsinglecontact} and \ref{fig:Automaton}.  On each phase, the continuous dynamics are represented as shown in Eq.~(\ref{eq:dynx}) and over a domain $\mathcal{D}(q)$.  We characterize the hybrid system as a directed graph $(\mathcal{Q},\mathcal{E})$ (see Fig.~\ref{fig:Automaton}), with nodes represented by $q\in\mathcal{Q}$ and \cools{edges} represented by $\mathcal{E}(q,q+1)$, that characterize the transitions between nodes. The transitions between states can be grouped into eight classes depending on whether a vector field or variable changes discontinuously and what the trigger mechanism is.  Table~\ref{table:SwitchingGuard} shows the transition classification. 
\begin{table}[ht]
 \caption{Transition Classifications.  System vector field is $\boldsymbol{\mathcal{F}}$ as shown in Eq.~(\ref{eq:accel}).}
  \begin{center}\vspace{-5mm}
   \begin{tabular}{|c|c|c|c|} \hline
    {\bf Type}        & {\bf Transition}     & {\bf Switching} &          {\bf Jump}             \\     \hline 
    Autonomous  & $\Delta_a^{[\tau]}$ 
                & $\quad\boldsymbol{\mathcal{F}_x}^{+}(\cdot\;,\; \cdot\;, \boldsymbol{x}^{+}, \;\cdot\;,\; \cdot) \leftarrow \Delta_a^{[\delta_s]}(\boldsymbol{x}^{-})$
                & $\quad\boldsymbol{x}^{+} \leftarrow \Delta_a^{[\delta_j]}(\boldsymbol{x}^{-})$\\ 
    \hline 
    Controlled  & $\Delta_c^{[\tau]}$ 
                & $\quad\boldsymbol{\mathcal{F}_x}^{+}(\cdot\;,\; \cdot\;,\; \cdot\;,\; \boldsymbol{u_x}^{+},\; \cdot) \leftarrow \Delta_c^{[\delta_s]}(\boldsymbol{u_x}^{-})$
                & $\quad\boldsymbol{u_x}^{+} \leftarrow \Delta_c^{[\delta_j]}(\boldsymbol{u_x}^{-})$\\ 
    \hline 
    ``Timed''     & $\Delta_t^{[\tau]}$ 
                & $\quad\boldsymbol{\mathcal{F}_x}^{+}(\zeta, \;\cdot\;,\; \cdot\;,\; \cdot\;, \;\cdot\;, \;\cdot) \leftarrow \Delta_d^{[\delta_t]}(\zeta)$
                & $\quad\boldsymbol{x}^{+} \leftarrow \Delta_t^{[\delta_j]}(\zeta)$\\ 
    \hline 
    ``Disturbed''    & $\Delta_d^{[\tau]}$ 
                & $\quad\boldsymbol{\mathcal{F}_x}^{+}(\cdot\;,\; \cdot\;,\; \cdot\;, \;\cdot\;, w_d) \leftarrow \Delta_d^{[\delta_s]}(w_d)$
                & $\quad\boldsymbol{x}^{+} \leftarrow \Delta_d^{[\delta_j]}(w_d)$\\ 
    \hline 
   \end{tabular}
  \end{center}
  \label{table:SwitchingGuard}
\end{table}

\noindent The hybrid automaton state is given by: $\boldsymbol{s}=(\zeta, q, \boldsymbol{x}^T)^T$.  $\tau\in\{\delta_s,\delta_j\}$ represents the ``switching'' or ``jump'' transition types respectively.  Detailed definitions for these transitions can be found in [\cite{branicky1998unified}]. The condition that triggers the type of \cool{event} (switching or jump) is determined by a \cools{guard} $\mathcal{G}(q, q+1)$ for the particular edge $\mathcal{E}_{q,q+1}$. With this information, let us formulate a robust hybrid automaton for our locomotion planning. 
\begin{figure}[t]
 \centering
   \includegraphics[width=.7\linewidth]{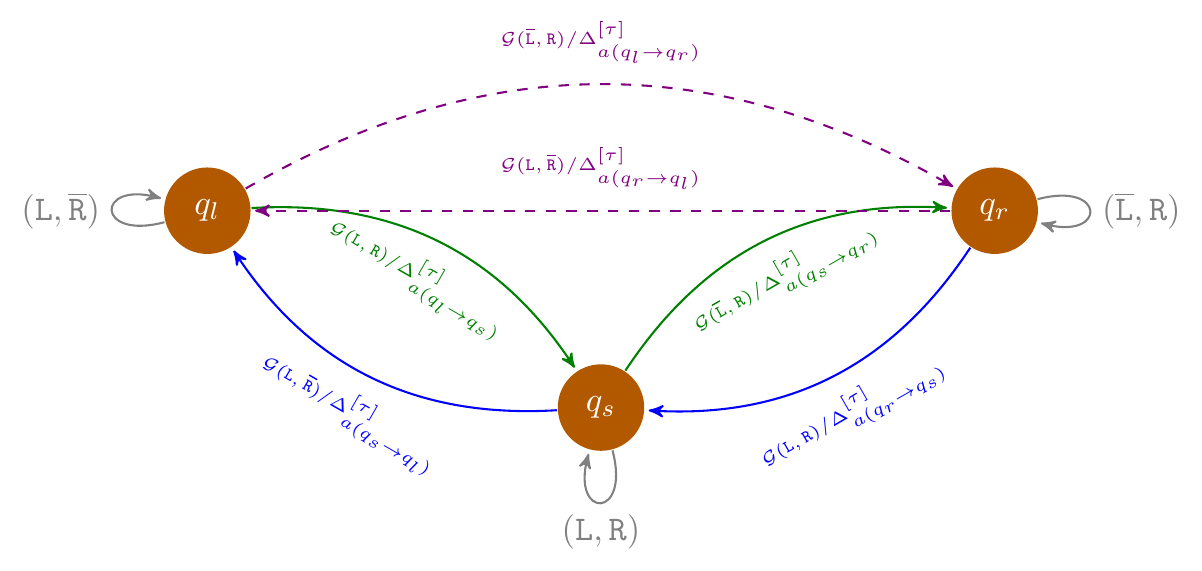}
 \caption{\captionsize This figure shows the hybrid locomotion automaton for a walking biped robot.  This automaton has three generic continuous modes, $\mathcal{Q}=\{q_l,q_s,q_r\}$, that represent when the robot is standing in the left leg only ($q_l$), standing in the right leg only ($q_r$), and when the robot is standing in both legs ($q_s$).  Shown in the edges are the condition for the transition, the \cools{guard} $\mathcal{G}(q, q+1)$ and effect of the transition in the state and outputs, the \cools{map} $\Delta_\mu^{[\tau]}$.  Note that, each mode is non-periodic.}
\label{fig:Automaton}
\end{figure}
\begin{definition}\label{def:PSRHA}
A phase-space robust hybrid automaton is a dynamical system, described by a $n$-tuple 
\begin{equation}\label{eq:PSRH}
{\rm PSRHA} \coloneqq (\zeta, \mathcal{Q},\mathcal{X}, \mathcal{U}, \mathcal{W}, \mathcal{F}, \mathcal{I},\mathcal{D}, \mathcal{R}, \mathcal{B}, \mathcal{E}, \mathcal{G}, \mathcal{T}, \Delta),
\end{equation}
\end{definition}
\noindent where, $\zeta$ is the phase-space progression variable, $\mathcal{Q}$ is the set of discrete states, $\mathcal{X}$ is the set of continuous states, $\mathcal{U}$ is the set of control inputs, $\mathcal{W}$ is the set of disturbances, $\mathcal{F}$ is the vector field, $\mathcal{I}$ is the initial condition, $\mathcal{D}$ is the domain, $\mathcal{R}$ is the collection of recoverability sets, $\mathcal{B}$ is the collection of invariant bundles, $\mathcal{E}\coloneqq \mathcal{Q}\times\mathcal{Q}$ is the edge, $\mathcal{G}: \mathcal{Q}\times\mathcal{Q}\rightarrow2^\mathcal{X}$ is the \cools{guard}, $\mathcal{T}$ is the transition termination set, and $\Delta$ is the transition map. 
\subsection{Step Transition Strategy}
\label{subsec:StepTrans}
Step transitions could be idealized as an instataneous contact (Fig.~\ref{fig:PS-stepTransition} (a)) or have a short dual-contact phase (Fig.~\ref{fig:PS-stepTransition}(b)). We first create a strategy for instantaneous contact switch then extend it to multi-contact in Appendix~\ref{sec:multicontact}.
\begin{definition}[A Phase-Space Simple Walking Step]\label{def:simple-walking-step} A $q^{\rm th}$ simple walking step is defined as a phase-space trajectory in domain $\mathcal{D}(q)$, which has two boundary guards $\mathcal{G}(q-1, q)$ and $\mathcal{G}(q, q+1)$.
\end{definition}
\noindent To characterize the non-periodic stability associated with walking in rough terrains, we define a progression map between apex states.

\begin{definition}[Stability of Non-Periodic Gaits]\label{def:snpg} We define the stability of non-periodic gaits as the progression map, $\Phi$, that takes the robot's center of mass from one desired apex state, $(\dot{x}_{{\rm apex}_q}, z_{{\rm apex}_q})$, to the next one on the phase-space manifold, and via the control input, $\boldsymbol{u}$, i.e.
\begin{equation}\label{eq:mapping}
(\dot{x}_{{\rm apex}_{q+1}}, z_{{\rm apex}_{q+1}}) = \Phi(\dot{x}_{{\rm apex}_q}, z_{{\rm apex}_q}, \boldsymbol{u_x}).
\end{equation}
\end{definition}
\begin{figure}[t]
 \centering
\includegraphics[width=0.95\textwidth]{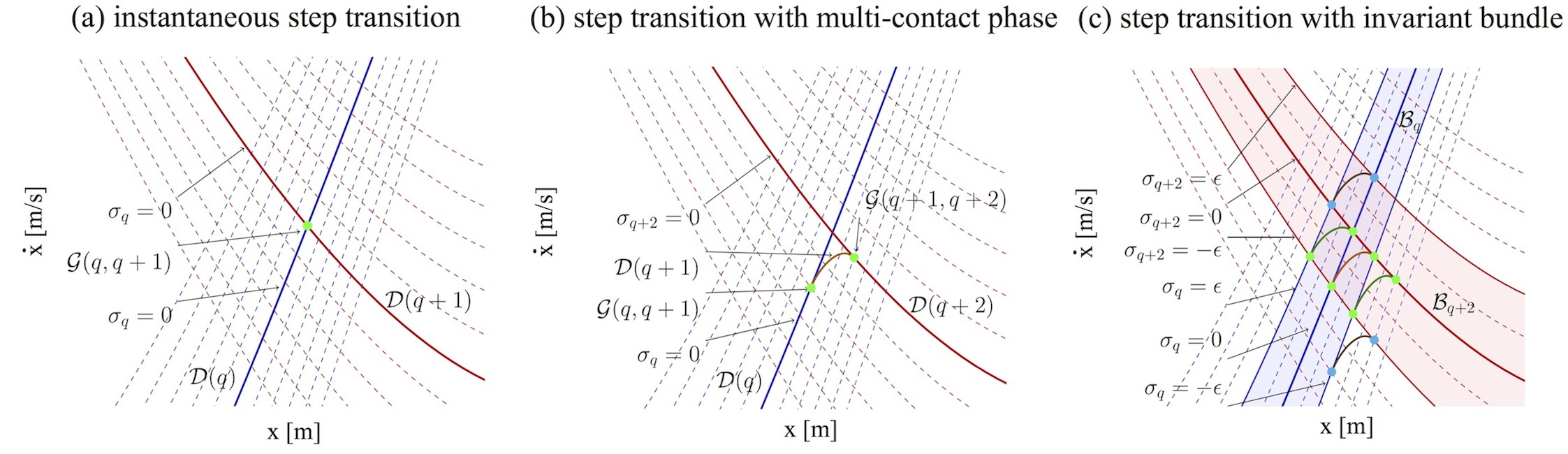}
 \caption{\captionsize Step transitions. This figure illustrates three types of step transitions in the Sagittal phase-space, associated with $\sigma$-isolines: instantaneous step transition in (a), step transition with multi-contact phase in (b) and step transition with invariant bundle in (c). (a) switches between two single contacts instantaneously while (b) has a multi-contact phase. (c) shows several guard alternatives for multi-contact transitions, from the current single-contact manifold value $\sigma_q$ to the next single-contact step bundle, $\sigma_{q+2}$.  In particular the invariant bundle bounds, $\sigma_{q} = \pm\epsilon$ are shown.  The transition bundle in green reattaches to the original manifold, $\sigma_{q+2}=0$, while the transition bundle (in brown) maintains its $\sigma$ value, i.e., $\sigma_{q+2}=\sigma_q$.}
 \label{fig:PS-stepTransition}
\end{figure}
\begin{definition}[Phase Progression Transition Value] \label{def:progVariable}
A phase progression transition value $\zeta_{\rm trans} : \mathcal{Q} \times \mathcal{X} \rightarrow \mathbb{R}_{\geq 0}$ is the value of the phase progression variable when the state $\boldsymbol{x}_q$ intersects a guard $\mathcal{G}$, i.e.,
\begin{align}
\zeta_{\rm trans} \coloneqq \inf \{\zeta > 0 \quad {\large|} \quad 
                  \boldsymbol{x}_{q} \in \mathcal{G}\}.
\end{align}
\end{definition}

\noindent We propose an algorithm to find transitions between adjacent steps, which occur at $\zeta_{\rm trans}$. We first consider the case of \emph{simple walking steps} as defined in Def.~\ref{def:simple-walking-step}. Given known step locations and apex conditions, phase-space curves can be numerically obtained using Algorithm 1. Phase-space curves in the general case have infinite slopes when crossing the zero-velocity axis. Therefore we fit NURBS (non-uniform rational B-splines)\footnote{Different from polynomials, non-rational splines or B\'ezier curves, NURBS can be used to precisely represent conics and circular arcs by adding weights to control points.} to the manifolds of the generated data (see Fig. \ref{fig:PSMSurf} for an illustration of adjacent step manifolds). Subsequently, finding step transitions just consists on finding the root difference between adjacent NURBS, which reduces to a simple polynomial root-finding problem. The pipeline to find step intersections is shown below.
\begin{center}
   \includegraphics[width=0.8\textwidth]{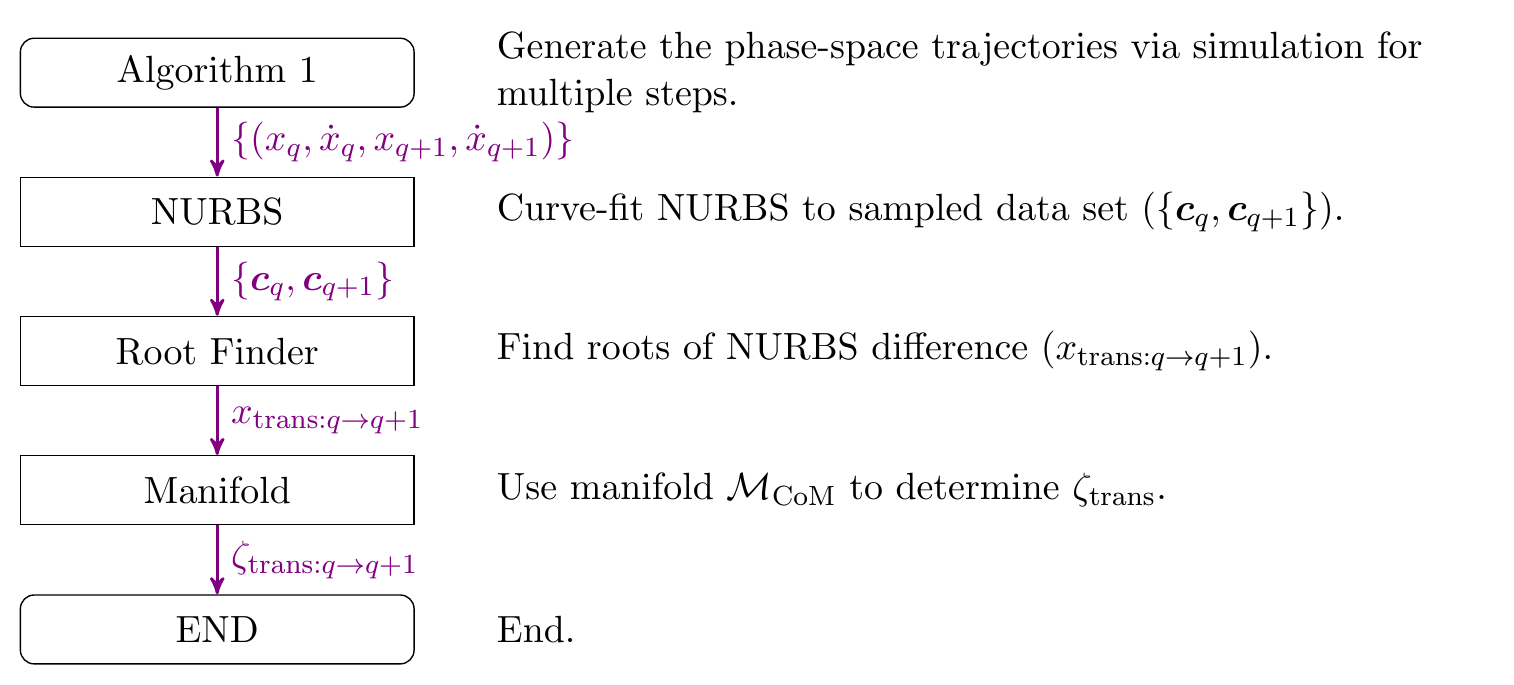}
 \end{center}
\subsection{Lateral Foot Placement Algorithm}
\label{subsec:searchlateral}
To complete the 3D walking planner, we formulate a searching strategy for lateral foot placement that complies with the timing of Sagittal step transitions. The main objective of the lateral dynamics is to return the robot's center of mass to a walking center through a semi-periodic cycle. If lateral foot placements are not adequately picked, the lateral behavior will drift away or become unstable. According to Eq.~(\ref{eq:accel}), lateral center of mass dynamics are equal to
\begin{equation}\label{eq:dyny}
\boldsymbol{\dot y} = \boldsymbol{\mathcal{F}_y}(\boldsymbol{y}, \boldsymbol{u_y})=
\begin{pmatrix}
\dot y\\
\omega^2 (y - y_{\rm foot}) - \frac{ \omega^2}{mg}\tau_y
\end{pmatrix},
\end{equation}
which can be numerically simulated adapting Algorithm 1 to the lateral dynamics (see Fig.~\ref{fig:Numerical_Integration} for simulations of lateral dynamics).
\setcounter{algorithm}{1}
\begin{algorithm}
\begin{algorithmic}[1]
\setstretch{1.3}
\STATE Initialize iteration index $n \leftarrow 1$, maximum iterations $n_{\rm max}$, tolerance $\dot{y}_{\rm tol}$ and initial state $y_{\rm init}, \hat{y}_{\rm foot}(1)$
\STATE  $\dot{y}_{\rm apex}(1) \leftarrow$ integration of inverted pendulum model given in Eq.~(\ref{eq:dyny}) with $\hat{y}_{\rm foot}(1)$
\WHILE {$n < n_{\rm max}$ and $|\dot{y}_{\rm apex}(n)| >\dot{y}_{\rm tol}$} 
\STATE $\quad \hat{y}_{\rm foot}(n+1) = \hat{y}_{\rm foot}(n) - \dot{y}_{\rm apex}(n)/\ddot{y}_{\rm apex}(n)$ by Newton-Raphson method
\STATE $\quad \dot{y}_{\rm apex}(n+1) \leftarrow$ integration of inverted pendulum in Eq.~(\ref{eq:dyny}) with $\hat{y}_{\rm foot}(n+1)$
\STATE $\quad \ddot{y}_{\rm apex}(n+1) = (\dot{y}_{\rm apex}(n+1) - \dot{y}_{\rm apex}(n))/(\hat{y}_{\rm foot}(n+1) - \hat{y}_{\rm foot}(n))$
\STATE $\quad n \leftarrow n +1$
\ENDWHILE
\end{algorithmic}
\caption{Newton-Raphson Search for Lateral Foot Placement}
\label{al:newtonraphson}
\end{algorithm}
To generate bounded lateral trajectories, we choose the simple criterion of achieving zero apex lateral velocity, $\dot{y}_{\rm apex} = 0$ at the instant when the CoM lateral apex position, $y_{\rm apex}$ is located between the two feet. Here $y_{\rm apex}$ and $\dot{y}_{\rm apex}$ are the lateral center of mass position and velocity when the center of mass crosses the Sagittal apex as defined in Def. \ref{def:sagapex}. Algorithm \ref{al:newtonraphson} achieves this objective. In this algorithm, $\hat{y}_{\rm foot}(n)$ represents the estimated lateral foot placements in the $n^{\rm th}$ search iteration. A foot placement range constraint $\hat{y}_{\rm foot, {\rm min}} \leq \hat{y}_{\rm foot} \leq \hat{y}_{\rm foot, {\rm max}}$ and the maximum iteration step constraint $n < n_{\rm max}$ are also provided. Examples of usage are given in Fig.~\ref{fig:Numerical_Integration}.

\begin{figure*}
 \centering
   \includegraphics[width=.95\linewidth]{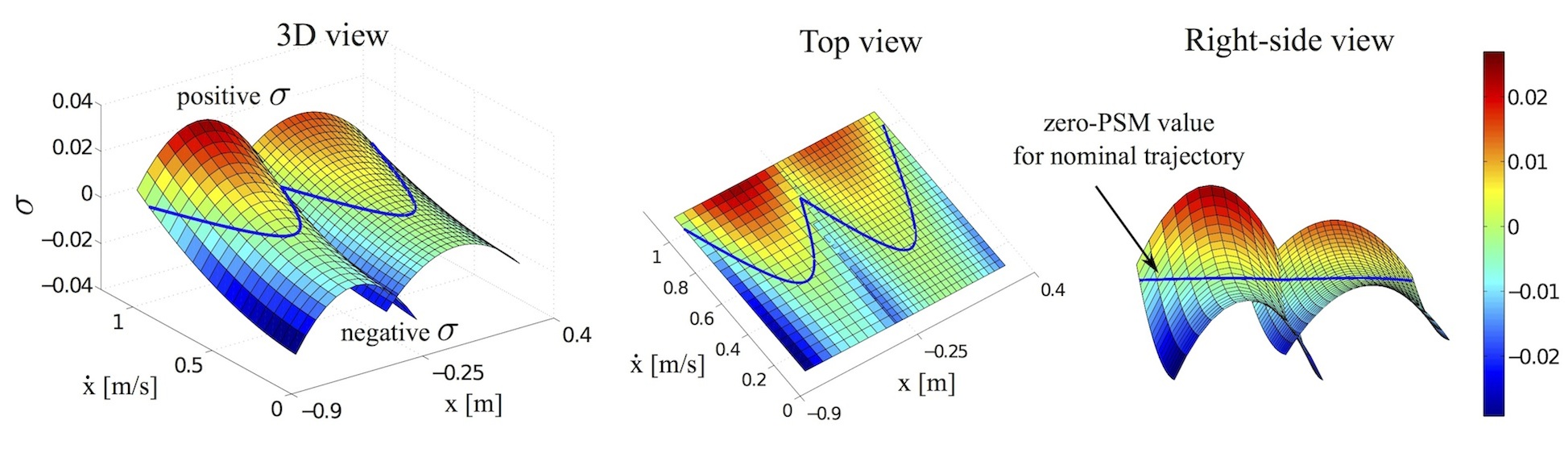}
 \caption{\captionsize Phase-space manifold isolines. This three-dimensional space demonstrates our phase-space manifold isolines defined in Eq.~(\ref{eq:surface1}) by the color map. The horizontal plane represents the Sagittal phase-space while the vertical third dimension represents the non-zero $\sigma$ value in Eq.~(\ref{eq:surface1}). As we can see, the blue nominal trajectory has a zero $\sigma$ value. The phase space region above the nominal trajectory has positive $\sigma$ values while the lower region has negative $\sigma$ values.}
 \label{fig:PSMSurf}
\end{figure*}

 \begin{figure}[t]
  \begin{center}
    \includegraphics[width=0.45\linewidth,
                     height=0.31\linewidth]{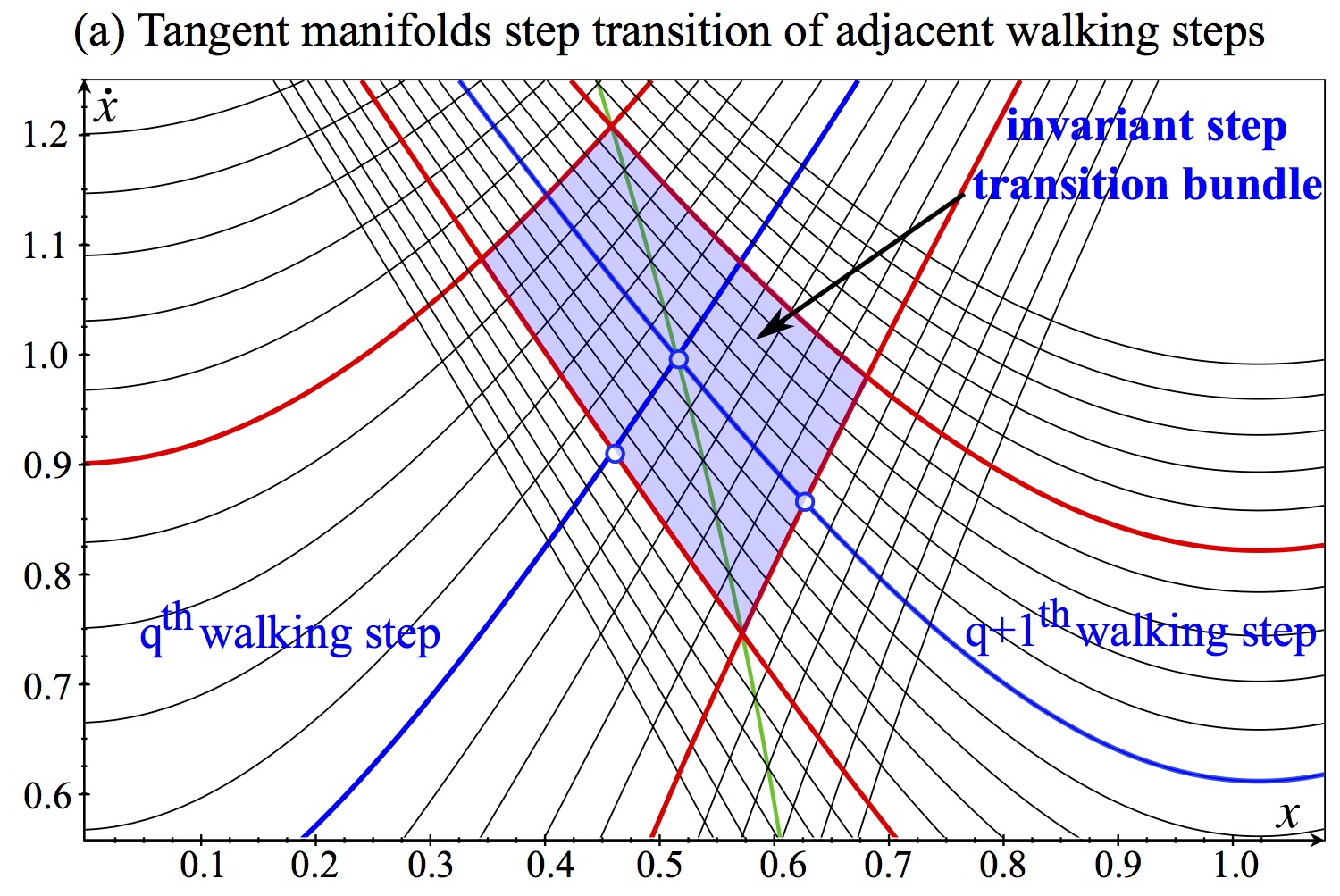}
    \includegraphics[width=0.435\linewidth]{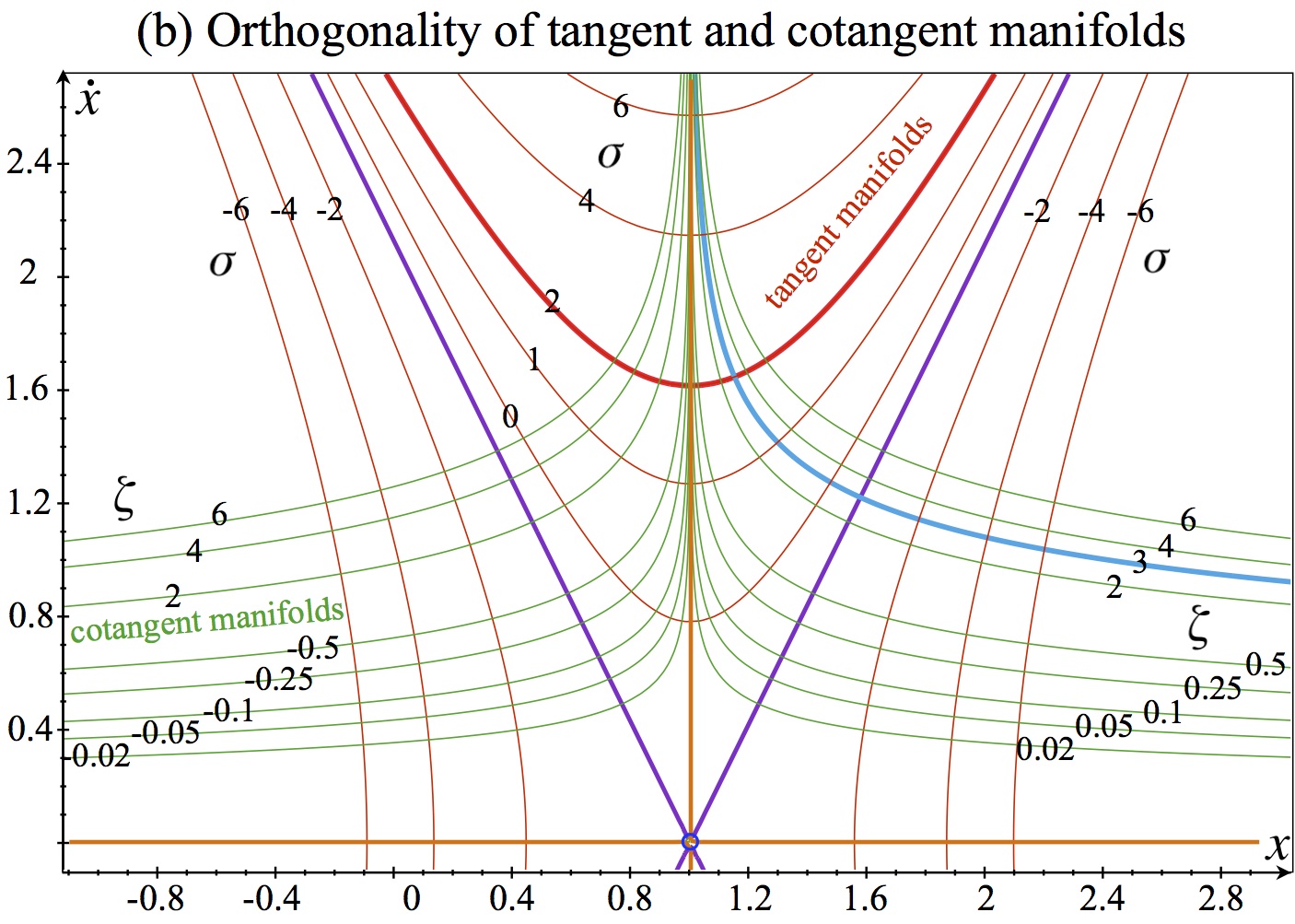}
 \caption{The left subfigure shows the tangent manifolds for two simple walking steps.  Nominal phase space manifolds are shown in blue color.  Their intersection corresponds to $\zeta_{\rm trans}$ and the guard, $\mathcal{G}_{q \rightarrow q+1}$.  Shown in red are manifolds of $\sigma=\pm \epsilon$.  For the current $q^{\rm th}$ step, we can use the $-\epsilon$ manifold of the next bundle as the guard $\mathcal{G}_{q\rightarrow q+1}=\{(x, \dot{x}) \;\big|\; \sigma_{q+1}=-\epsilon\}$.  The subfigure on the right shows the orthogonal tangent and cotangent manifolds (\textit{isoparametric curves}) in phase-space. The tangent manifold (red lines) are shown as lines of constant $\sigma$ as defined in Eq.~(\ref{eq:simplifiedPSM}). Thick lines in purple are the asymptotes of tangent manifolds and the thick red line represents the specific manifold with $\sigma=2$. The asymptotes intercept the $(x_{\rm foot}, 0)$ saddle point, where $x_{\rm foot}$ is the Sagittal foot position.  The point noted with a circle is the origin of the $(\zeta,\sigma)$ (transformed) space. The cotangent manifold (green lines) are lines of constant $\zeta$ that are orthogonal at every point to the constant $\sigma$-manifolds.  Thick lines in orange are the asymptotes of cotangent manifolds and the thick cyan line represents a specific manifold with $\zeta=3$.  The vertical asymptote represents the manifold with $\zeta=0$.}
   \label{fig:Detailed-2Half-Steps1}
  \end{center}
 \end{figure}

\section{Phase-Space Manifold Analysis}
\label{sec:manifold}
In this section, we focus on formulating a metric to measure deviations from the phase-space manifolds of planned trajectories derived in the previous section. A phase-space sensitivity norm is also formulated to study the effect of disturbances. Various disturbance patterns and suggested recovery strategies are considered. 
\begin{proposition}[Phase-Space Tangent Manifold]\label{theorem:PSM}
Given the prismatic inverted pendulum dynamics of Eq. (\ref{eq:dynx}) with an initial condition $(x_0, \dot{x}_0)$ and foot placement $x_{\rm foot}$, the phase-space tangent manifold is
\begin{align}\label{eq:surface1}
\sigma =  \; (x_0 - x_{\rm foot}) ^2\Big(2\dot{x}^2_0 - \dot{x}^2 + \omega^2 (x - x_0) (x + x_0 - 2x_{\rm foot})\Big)
- \dot{x}^2_0 (x - x_{\rm foot})^2 + \dot{x}^2_0 (\dot{x}^2 - \dot{x}^2_0)/\omega^2,
\end{align}
where the condition $\sigma = 0$ is equivalent to the nominal phase-space manifold. Furthermore, $\sigma$ represents the Riemanniam distance to the estimated locomotion phase-space trajectories.
\end{proposition}
\begin{proof}
Given in the Appendix~\ref{sec:PSMDerivation}.
\end{proof}
\noindent If we use the apex conditions as initial values, i.e. $(x_0, \dot{x}_0) = (x_{\rm foot}, \dot{x}_{\rm apex})$, the tangent manifold becomes
\begin{align}\label{eq:simplifiedPSM}
\sigma(x, \dot{x}, \dot{x}_{\rm apex}, x_{\rm foot}) = \dfrac{\dot{x}^2_{\rm apex}}{\omega^2} \Big(\dot{x}^2 - \dot{x}^2_{\rm apex} 
                                            - \omega^2(x - x_{\rm foot})^2\Big).
\end{align}
Since the tangent manifold is considered as a trajectory deviation metric in the phase-space, we will use it in the next section as a feedback control parameter to ensure robustness. Fig.~\ref{fig:PSMSurf} provides an illustration of the value of $\sigma$ as a function of the state. The same type of analysis can be done for lateral trajectory deviations using the lateral pendulum dynamics of Eq. (\ref{eq:dyny}).

\begin{figure*}[t]
 \centering
   \includegraphics[width=\linewidth]
   {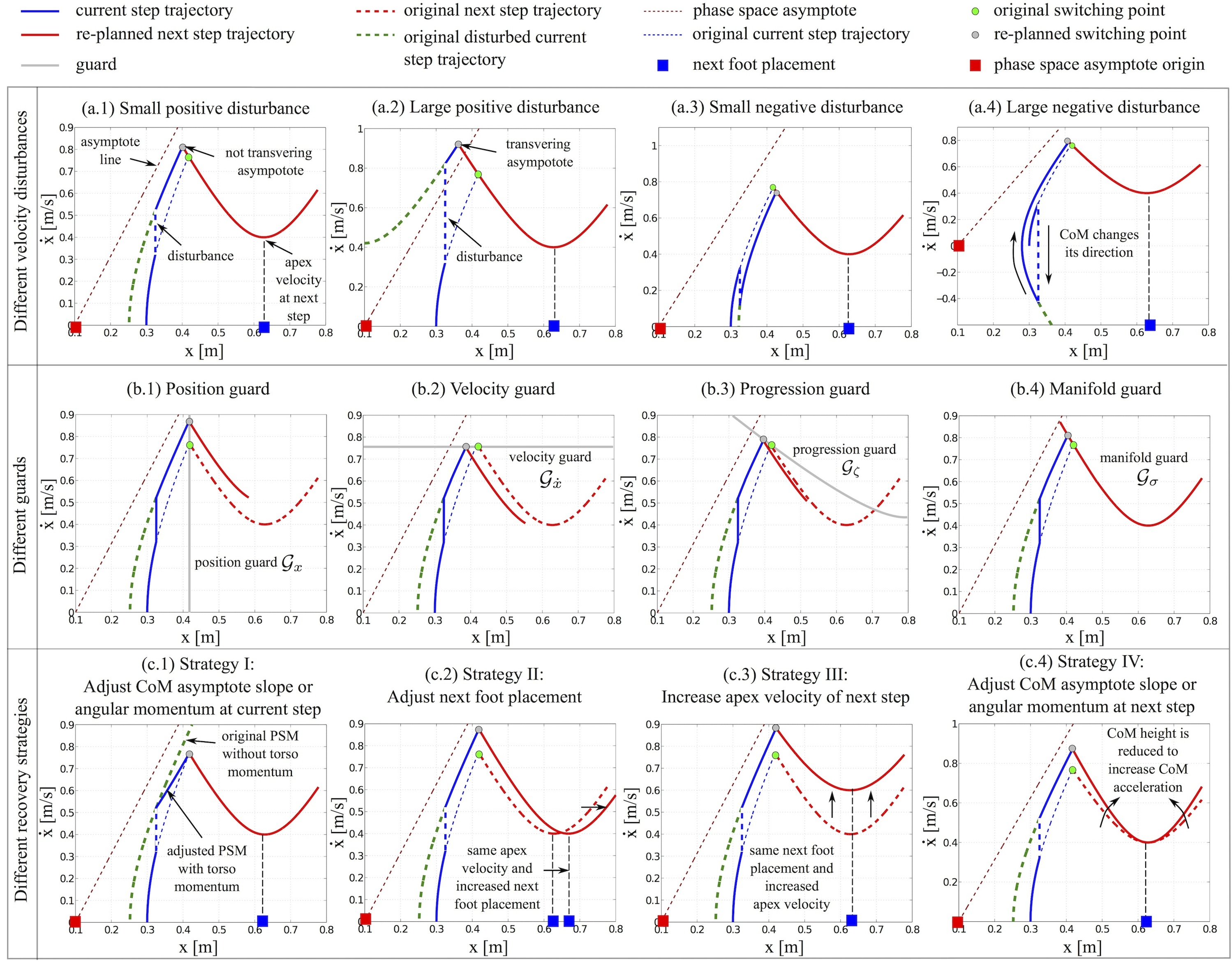}
 \caption{\captionsize Disturbance pattern, guard and recovery strategy classification. Four different velocity disturbance cases are shown in subfigures (a.1)-(a.4). The second row shows four proposed guards for the next step transition while the last row provides four recovery strategies.}
 \label{fig:VelocityDisturbSet1}
\end{figure*}

\begin{proposition}[Phase-Space Cotangent Manifold]\label{prop:PSCoM}
Given the pendulum system of Eq.~(\ref{eq:dynx}), the cotangent manifold is equal to
\begin{align}\label{eq:zeta_manifold}
\zeta  = \zeta_0(\dfrac{\dot{x}}{\dot{x}_0})^{\omega^2} \dfrac{x - x_{\rm foot}}{x_0 - x_{\rm foot}},
\end{align}
and represents the arc length along the tangent manifold of Eq.~(\ref{eq:simplifiedPSM}), with an initial value $\zeta_0$.
\end{proposition}
\begin{proof}
Given in the Appendix~\ref{sec:OrthogonalManifoldDerivation}.
\end{proof}
\noindent Illustration of the tangent manifold and cotangent manifold are given in Fig. \ref{fig:Detailed-2Half-Steps1}.
\noindent In robust control theory [\cite{zhou1996robust}], input-output system response can be evaluated through the use of system norms. In this spirit, we define a new norm that characterizes sensitivity to disturbances of our non-periodic gaits, as
\begin{definition}[Phase-Space Sensitivity Norm]\label{def:sensitivityNorm}
Given a disturbance $d$, the phase-space sensitivity norm is defined as
\begin{align}
\kappa\Big(\sigma(x_{\zeta_d}, \dot{x}_{\zeta_d})\Big) = \Big(\dfrac{1}{\zeta_{\rm trans} - \zeta_d} \int_{\zeta_d}^{\zeta_{\rm trans}}\sigma(x_\zeta, \dot{x}_\zeta)^2 d \zeta\Big)^{1/2},
\end{align}
where, $\zeta_d$ corresponds to the phase value when a disturbance occurs and $\zeta_{\rm trans}$ is the phase transition defined in Def.~\ref{def:progVariable} for a given step. 
\end{definition}
\noindent Contrary to other sensitivity norms [\cite{hobbelen2007disturbance, hamed2015robust}], our gait norm evaluates disturbance sensitivity for non-periodic gaits. It does so by explicitly accounting for disturbance magnitude and for the instant where disturbances occur. And it does not rely on approximate linearization nor Taylor series expansion as the previous periodic gait norms require. We will use the proposed norm in the control section for dynamic programming. 

Let us consider various types of disturbances and corresponding recovery strategies. Disturbances can be categorized based on four characteristics: (1) the disturbance direction, (2) the disturbance magnitude, (3) the terminal asymptote-region, and (4) the change of the motion direction.  Fig.~\ref{fig:VelocityDisturbSet1} (a.1)-(a.4) illustrates those four scenarios, respectively. Here, we assume that the disturbances are impulses that change the velocity instantly\footnote{The disturbance could also be impulses changing the position or a combination of both position and velocity. In any case, the proposed disturbance characteristics and recovery strategies are still valid.}.   (a.2) has a larger positive disturbance than (a.1) such that velocity after the disturbed trajectory crosses the asymptote of the inverted pendulum model. On the other hand, (a.3) has a smaller negative disturbance such that velocity after disturbance keeps the same direction while (a.4) does not. More disturbance scenarios could be defined, depending on specific phase-space state locations and disturbance characteristics. Given these disturbed phase-space trajectories, new step transition strategies need to be considered. Here we propose four types of guard strategies for next step transition in Fig.~\ref{fig:VelocityDisturbSet1} (b.1)-(b.4). The guards shown are: position guard $\mathcal{G}_x$ (vertical line), velocity guard $\mathcal{G}_{\dot{x}}$ (horizontal line), progression guard $\mathcal{G}_\zeta$ ($\zeta$-isoline), and manifold guard $\mathcal{G}_\sigma$ ($\sigma$-isoline).  We find each guard such that they have the same transition point for the nominal phase space manifold (PSM). Although this guard recovering strategy forces the motion to be adjusted, it might not be sufficient. If that is the case, we consider designing more recovery strategies by properly using control inputs. In the last four subfigures of Fig.~\ref{fig:VelocityDisturbSet1}, four recovery strategies are illustrated. These strategies are inspired by observations of human walking behaviors [\cite{hofmann2006robust, kuo1992human, abdallah2005biomechanically}] and by the experiences gained during extensive simulations. In order to fulfill those recovery strategies, a proper control policy will need to be designed.

\section{Hybrid Control Strategy under Disturbances}
\label{sec:optimization}

This section formulates a two-stage control procedure to recover from disturbances.  When a disturbance occurs, the robot's CoM deviates from the planned phase-space manifolds obtained via Algorithm 1.  We use dynamic programming to find an optimal policy of the continuous control variables for recovery, and, when necessary, we re-plan feet placements from their initial locations based on the guards defined in Fig.~\ref{fig:VelocityDisturbSet1}. Our proposed controller, relies on the distance metric of Eq.~(\ref{eq:simplifiedPSM}) to steer the robot current's trajectory to the planned manifolds.  

\subsection{Dynamic Programming-Based Optimal Control}
\label{subsec:DP}
This subsection shows a proposed dynamic programming-based controller for the continuous control of the Sagittal locomotion dynamics.  A similar controller can be formulated for the lateral CoM behavior. To robustly track the planned CoM manifolds, we minimize a finite phase quadratic cost function and solve for the continuous control parameters, i.e.
\begin{equation}\label{eq:optimization-1}
\begin{aligned}
&\underset{\boldsymbol{u}_{\boldsymbol{x}}^c}{\text{min}} \;\; \mathcal{V}_N(q, \;\boldsymbol{x}_N) + \sum_{n=0}^{N-1} \eta^n \mathcal{L}_n(q, \boldsymbol{x}_n, \boldsymbol{u}^c_{\boldsymbol{x}})& \\
&\textrm{subject to}:\; \boldsymbol{\dot x} = \boldsymbol{\mathcal{F}_x}(\boldsymbol{x}, \boldsymbol{u}^c_{\boldsymbol{x}}, d), \\
&\hspace{0.7in}\omega^{\rm min} \leq \omega \leq \omega^{\rm max}, \\
&\hspace{0.7in} \tau^{\rm min}_y \leq \tau_y \leq \tau^{\rm max}_y,
\end{aligned}
\end{equation}
where $\boldsymbol{u}_{\boldsymbol{x}}^c = \{\omega, \tau_y\}$ corresponds to the continuous variables of the hybrid control input $\boldsymbol{u_x}$ of Eq. (\ref{eq:dynx}), $0 \leq \eta \leq 1$ is a discount factor, $N$ is the number of discretized stages until the next step transition, $\zeta_{\rm trans}$, the terminal cost is, $\mathcal{V}_N(q, \boldsymbol{x}_N) = \alpha (\dot{x}(\zeta_{\rm trans}) - \dot{x}(\zeta_{\rm trans})^{\rm nom})^2$. Here, $\dot{x}(\zeta_{\rm trans})$ is the disturbed terminal velocity which happens to be chosen at the instant of the next step transition, and $\dot{x}(\zeta_{\rm trans})^{\rm nom}$ is the nominal transition velocity. The first equality constraint $\boldsymbol{\mathcal{F}_x}(\cdot)$ is defined by the PIPM dynamics of Eq. (\ref{eq:dynx}) with an extra input disturbance $d$. Additionally, $\mathcal{L}_n$ is the one step cost-to-go function at the $n^{\rm th}$ stage defined as
\begin{align} \label{eq:L_s}
\mathcal{L}_n(q, \boldsymbol{x}_n, \boldsymbol{u}^c_{\boldsymbol{x}}) = & \int_{\zeta_n}^{\zeta_{n+1}} \big[\beta \sigma ^2 + \Gamma_1 \tau_y^2 + \Gamma_2 (\omega - \omega^{\rm ref})^2 \big] d\zeta,
\end{align}
where, $\sigma$ is the tangent manifold of Eq.~(\ref{eq:surface1}) used as a feedback control parameter, $\zeta_n$ and $\zeta_{n+1}$ are the starting and ending phase progressions for the $n^{\rm th}$ stage, $\alpha$, $\beta$, $\Gamma_1$ and $\Gamma_2$ are weights, and $\omega^{\rm ref}$ is the reference phase space asymptote slope given in Algorithm 1. To avoid chattering effects in the neighborhood of the planned manifold, a boundary layer is defined and used to saturate the controls, i.e.
\begin{subequations}
\label{eq:slidingcontrol}
\begin{empheq}[left={\boldsymbol{u}^{c'}_{\boldsymbol{x}} = }\empheqlbrace]{align}\label{eq:slidingcontrol_a}
  & \boldsymbol{u}^c_{\boldsymbol{x}} & |\sigma| > \epsilon \\\label{eq:slidingcontrol_b}
  & \dfrac{|\sigma|}{\epsilon}\boldsymbol{u}^{c, \epsilon}_{\boldsymbol{x}} + \dfrac{\epsilon - |\sigma|}{\epsilon}\boldsymbol{u_x}^{c, {\rm ref}} & |\sigma| \leq \epsilon
\end{empheq}
\end{subequations}
where $\epsilon$ corresponds to the boundary value of an invariant bundle $\mathcal{B}(\epsilon)$ as defined in Def.~\ref{def:invariantBundle}, $\boldsymbol{u}_{\boldsymbol{x}}^{c, \epsilon} = \{\omega^\epsilon, \tau_y^\epsilon\}$ are control inputs at the instant when the trajectory enters the invariant bundle $\mathcal{B}(\epsilon)$, $\boldsymbol{u}_{\boldsymbol{x}}^{c, {\rm ref}}$ are nominal control inputs defined in Algorithm 1. A proof of smoothness of the above control saturation function is discussed in [\cite{utkin2013sliding}]. As Eq.~(\ref{eq:slidingcontrol}) shows, when $|\sigma| \leq \epsilon$, the control effort, $\boldsymbol{u}^{c'}_{\boldsymbol{x}}$ is scaled between $\boldsymbol{u}_{\boldsymbol{x}}^{c, \epsilon}$ and $\boldsymbol{u}_{\boldsymbol{x}}^{c, {\rm ref}}$. This control law is composed of an ``inner'' and an ``outer'' controller. The ``outer'' controller steers states into $\mathcal{B}(\epsilon)$ while the ``inner'' controller maintains states within $\mathcal{B}(\epsilon)$.  Recovery trajectories are shown in Fig.~\ref{fig:DP_robust_bound} for two scenarios in the presence of random disturbances.
\begin{figure}[t]
\centering
\includegraphics[width=0.95\linewidth]{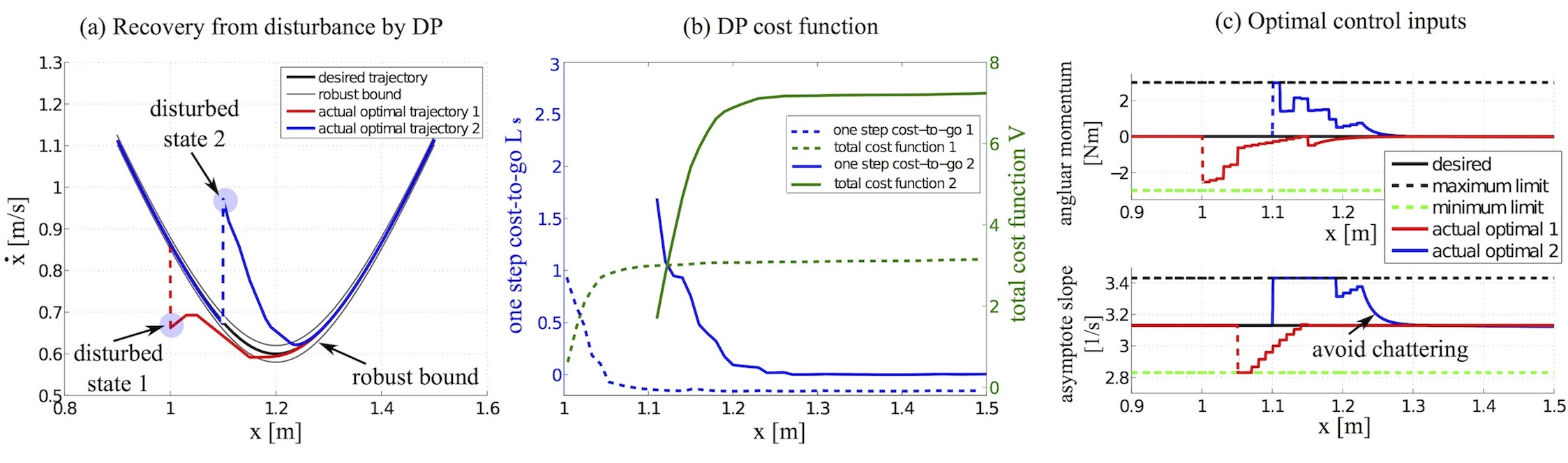}
\caption{\captionsize Chattering-free recoveries from disturbance by the proposed optimal recovery continuous control law. Subfigure (a) shows two random disturbances, where disturbed state 1 has a positive impulse while the disturbed state 2 has a negative impulse. Control variables are piece-wise constant within one stage as shown in subfigure (c). In these simulations, angular momentum control range is $[-3, 3]$ Nm and phase space asymptote slope range is $[2.83, 3.43]$ 1$/$s. Other parameters are shown in Table~\ref{table:DP_parameter}.}
\label{fig:DP_robust_bound}
\end{figure}

\begin{table}[hbt!]
\caption{Dynamic Programming Parameters} 
\begin{center}\vspace{-5mm}
\begin{tabular}{|c|c||c|c|}
\hline
Parameter & Range & Parameter & Range\\ \hline
desired momentum $\tau^{\rm ref}_y$ & 0 Nm& desired asymptote slope $\omega^{\rm ref}$ &3.13 1$/$s\\ \hline
angular momentum range $\tau_y^{\rm range}$ & [-3, 3] Nm& asymptote slope range $\omega^{\rm range}$ &[2.83, 3.43] 1$/$s\\ \hline
foot placement $x_{\rm foot}$ & 1.2 m & mass $m$ & 1 kg\\\hline
stage range & [0.9, 1.5] m & state range & [0.03, 1.5] m/s \\\hline
stage resolution $\rm stage_{\rm res}$ & 0.01 m & state resolution $\rm state_{\rm res}$ & 0.01 m/s \\\hline
disturbed initial state $s_{\rm initial}$ & $(1.1$ m ,$0.7$ m$/$s$)$ & desired apex velocity $\dot{x}_{\rm apex}$ & 0.6 m$/$s\\\hline
weighting scalar $\Gamma_{1}$ & 5 &weighting scalar $\Gamma_{2}$ & 5\\\hline 
weighting scalar $\beta$ & $4 \times 10^4$ &weighting scalar $\alpha$ & 100 \\\hline 
\end{tabular}
\end{center}
\label{table:DP_parameter}
\end{table}

Since the control inputs are constrained within a desired range, i.e. $\boldsymbol{u}^c_{\boldsymbol{x}}  \in \boldsymbol{u}^{c, {\rm range}}_{\boldsymbol{x}}$, we re-define the finite-phase control-dependent recoverability bundle.
Given an acceptable deviation $\epsilon_0$ from the manifold, the practical invariant bundle is $\mathcal{B}(\epsilon_0)$.
The control policy in Eq.~(\ref{eq:slidingcontrol}) generates a control-dependent practical recoverability bundle (a.k.a., region of attraction to the ``boundary-layer'') defined as 
 \begin{align}\label{eq:OptimalRecoverabilty}
  \mathcal{R}(\epsilon, \zeta_{\rm trans}) = \Big\{\boldsymbol{x}_{\zeta} \in \mathbb{R}^2, \quad \zeta_0\le\zeta\le \zeta_{\rm trans}\; \big| \; 
    \boldsymbol{x}_{{\zeta}_{\rm trans}} \in \mathcal{B}(\epsilon), \quad \boldsymbol{u}^c_{\boldsymbol{x}}  \in \boldsymbol{u}^{c, {\rm range}}_{\boldsymbol{x}} \Big\}. 
 \end{align}

\begin{theorem*}[Existence of Recoverability Bundle]\label{theo:theorem}
Given a Lyapunov function $V = \sigma^2/2$, a phase progression transition value $\zeta_{\rm trans}$, and the control policy in Eq. (\ref{eq:slidingcontrol}), a recoverability bundle $\mathcal{R}(\epsilon, \zeta_{\rm trans})$ exists and can be bounded by a maximum tube radius $\sigma^{\rm max}_0$.
\end{theorem*}

\begin{proof}
We will use a Lyapunov function to prove the existence of $\mathcal{R}(\epsilon, \zeta_{\rm trans})$. First, let us define $V = \sigma^2/2$ based on Eq.~(\ref{eq:slidingcontrol_a}).  If there exists a control policy such that $\exists\; \sigma_0 > \epsilon, \sigma_{\rm trans} \leq \epsilon$, then, $\mathcal{R}(\epsilon, \zeta_{\rm trans})$ is composed of the range of values $(x, \dot{x})_{\zeta}, \; \zeta_0\le\zeta\le \zeta_{\rm trans}$, such that $V_{\rm trans} = \sigma_{\rm trans}^2/2\le \;\epsilon^2/2$. Taking the derivative of $V$ along the pendulum dynamics in Eq.~(\ref{eq:accel}), we have
\begin{align}\nonumber
\dot{V} = \sigma\dot{\sigma} &= \sigma \dot{x}^2_{\rm apex} \Big( -2 \dot{x} (x - x_{\rm foot}) + 2 \dot{x} \ddot{x}/\omega^2\Big) = \sigma  \dot{x}^2_{\rm apex} \Big( -2 \dot{x} (x - x_{\rm foot}) + 2 \dot{x} \big((x - x_{\rm foot}) - \dfrac{\tau_y}{mg}\big) \Big)\\\label{eq:stability}
&= -\dfrac{2  \dot{x}^2_{\rm apex} \sigma \dot{x} \tau_y}{mg} = -\dfrac{2 \sqrt{2}  \dot{x}^2_{\rm apex} \dot{x} \tau_y \cdot {\rm sign}(\sigma)}{mg} \sqrt{V} \le 0.
\end{align}

\noindent which proves the stability (i.e., attractiveness) of $\sigma=0$. 
For example, consider the case of walking forward, $\dot{x} > 0$. Then, as long as $\sigma \cdot \tau_y > 0$, i.e., the pitch torque has the same sign as $\sigma$, attractiveness is guaranteed.  That is, if $\sigma>0$ (the robot moves forward faster than expected), then we need $\tau_y > 0$ to slow down, and vice-versa.  To estimate $\mathcal{R}(\epsilon,\zeta_{\rm trans})$, we can use the optimal control policy proposed in Eq.~(\ref{eq:slidingcontrol}), defining an ``optimal'' recoverability bundle; or (ii) the maximum control inputs (without any regards to optimality) obtained by selecting the bounds $\boldsymbol{u}^{c, {\rm range}}_{\boldsymbol{x}}$, defining the ``maximum'' recoverability bundle: 
\paragraph{Case I: DP based Control.}
 If $\tau_y$ is solved by the optimization of Eq.~(\ref{eq:optimization-1}), we get $|\tau_y| > |\tau^\epsilon_y|$. Then Eq.~(\ref{eq:stability}) becomes
\begin{align}\label{eq:recoverability-condition-1}
\dot{V} < -\dfrac{2 \sqrt{2}  \dot{x}^2_{\rm apex} \dot{x} |\tau^\epsilon_y|}{mg} \sqrt{V} <   0.
\end{align}
\paragraph{Case II: Supremum (Bang-Bang) Control.}
 If we design $\tau_y= \tau_y^{\rm max}\sign(\dot{x})$ and $\dot{x} > 0$, then,
\begin{align}\label{eq:recoverability-condition-2}
\dot{V} = -\dfrac{2 \sqrt{2}  \dot{x}^2_{\rm apex} \dot{x} \tau_y^{\rm max}\sign(\dot{x})}{mg} \sqrt{V} 
            = -\dfrac{2 \sqrt{2}  \dot{x}^2_{\rm apex} \dot{x} \tau_y^{\rm max}}{mg} \sqrt{V} < 0.
\end{align}
\begin{figure}[t]
 \centering
   \includegraphics[width=0.9\linewidth]{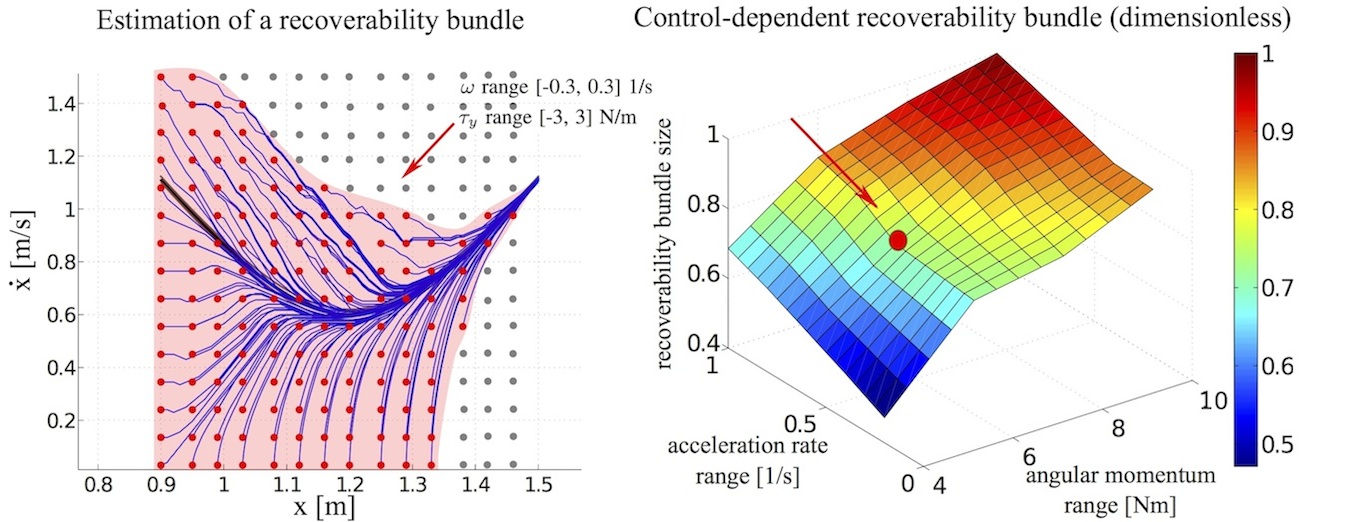}
 \caption{\captionsize Estimation of dimensionless control-dependent recoverable region. In the left figure, disturbed states are sampled in a discretized grid and the shaded region represents the recoverability bundle. As it is shown, a larger recoverable region is achieved in the beginning of the phase (i.e., before the apex state). In the ending phase, the recoverable region shrinks to the invariant bundle. Here the control constraint is: $\omega \in [-0.3, 0.3]$ 1/s and $\tau_y \in [-3, 3]$ N/m. The right figure shows the dependence of the size of the recoverable region with respect to the allowed control ranges. For visualization, the control range in the horizontal axes are labeled as values $r$, which implies the range is $[-r/2, r/2]$. 
 }
 \label{fig:RecoverabilitySet}
\end{figure}
\noindent Note that, the bounded $\dot{V}$ in Eqs.~(\ref{eq:recoverability-condition-1}) and (\ref{eq:recoverability-condition-2}) have similar forms and can be combined into a common integral equation of form
\begin{equation}\label{eq:integralLyapunov}
\int_{V_0}^{V_{\rm trans}} \dfrac{dV}{\sqrt{V}} = - \int_{t_0}^{t_{\rm trans}} \mu \dot{x} \tau_y dt = - \mu \tau_y(x_{\rm trans} - x_0),
\end{equation}
where $\mu = (2 \sqrt{2}  \dot{x}^2_{\rm apex})/(mg)$, $\tau_y = \tau_y^\epsilon$ for Case I while $\tau_y = \tau_y^{\rm max}$ for Case II. Eq.~(\ref{eq:integralLyapunov}) can be solved using common integral rules to yield 
\begin{equation}
\sqrt{V_0} = \sqrt{V_{\rm trans}} + \dfrac{1}{2}\mu \cdot (x_{\rm trans} - x_0) \cdot \tau_y.
\end{equation}
Since $V_0 = \sigma_0^2/2, V_{\rm trans} = \sigma_{\rm trans}^2/2 \leq \epsilon^2/2$, we get
\begin{align}
\quad \sigma_0 \leq \epsilon + \dfrac{\sqrt{2}}{2}\mu \cdot (x_{\rm trans} - x_0) \cdot \tau_y = \sigma^{\rm max}_0,
\end{align}
where $\sigma^{\rm max}_0$ defines the maximum tube radius. Therefore we can re-write the recoverability bundle of Eq.~(\ref{eq:OptimalRecoverabilty}) using this new tube radius as:
\begin{equation}\label{eq:practicalRecover}
\mathcal{R}(\epsilon, \zeta_{\rm trans}) = \Big\{\boldsymbol{x}_{\zeta} \in \mathbb{R}^2, \; \zeta_0\le\zeta\le \zeta_{\rm trans}\; \big| \; 
    \epsilon \leq \sigma_0^{\rm max}\Big\}.
\end{equation}
The existence of recoverability bundle is proved and a maximum tube radius is provided.
\end{proof}
\noindent Since Eq.~(\ref{eq:recoverability-condition-1}) has an inequality bound while Eq.~(\ref{eq:recoverability-condition-2}) has an equality bound, DP based control is an optimal but conservative estimation of the true recoverability bundle while supremum control is an accurate but non-optimal estimation for the recoverability bundle. 
\begin{figure}[!t]
\centering
\includegraphics[width=0.95\linewidth]{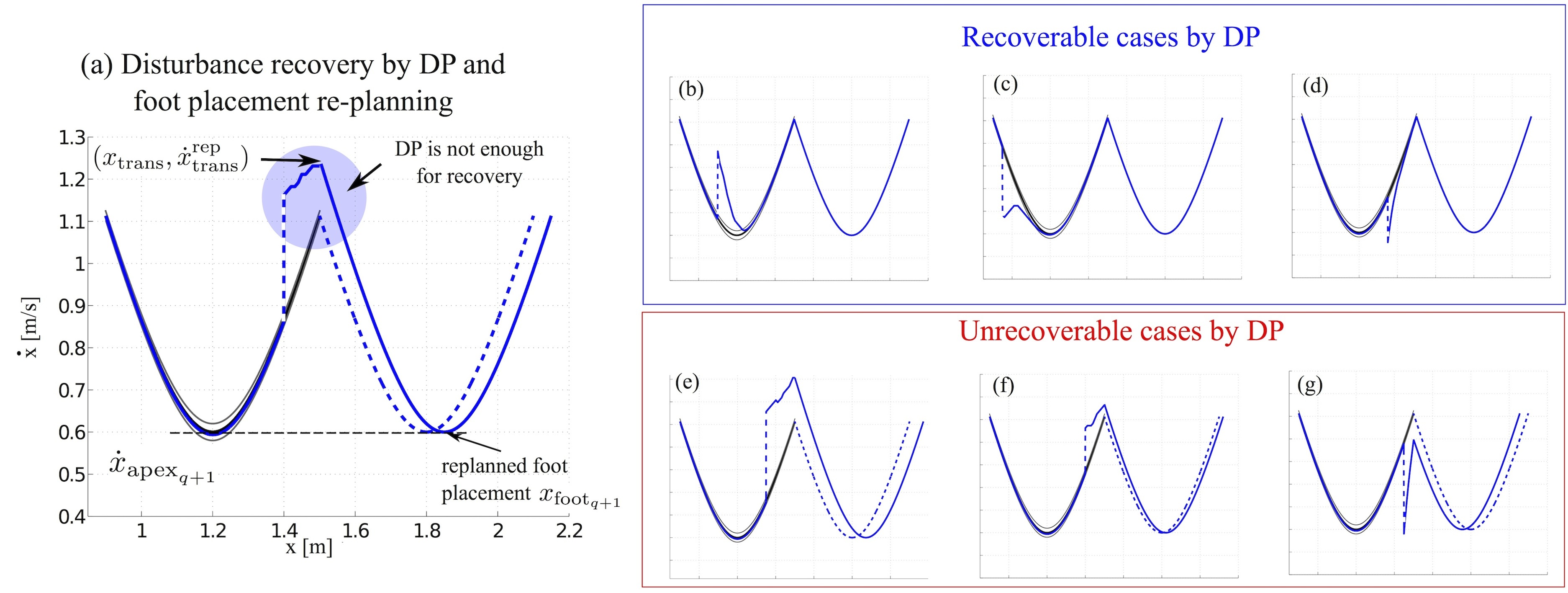}
\caption{\captionsize Recovery from a disturbance by re-planning Sagittal foot placement. In this case, the next apex velocity is given \textit{a priori} and maintained despite the disturbance. In subfigures (b)-(d), first-stage continuous DP control is sufficient to achieve the recovery while in the cases of subfigures (e)-(g) it is not. The latter cases occur when either disturbance occurs too close to the transition or is too large. In these cases, a new next foot placement is automatically re-planed based on Eq.~(\ref{eq:replanfoot}).}
\label{fig:High_Level_Robust}
\end{figure}
Our study aims at optimal performance, and therefore the control policy generated from dynamic programming will be used to estimate the recoverability bundle. 
To estimate $\mathcal{R}(\epsilon, \zeta_{\rm trans})$, we perform a grid sampling from the initial condition $\boldsymbol{x}_{\zeta_0}$, based on the ranges of Table~\ref{table:DP_parameter}. Then we execute the optimization of Eq.~(\ref{eq:optimization-1}) for each sampled $\boldsymbol{x}_{\zeta_0}$ (treated as a realization) and repeat this procedure for all $\boldsymbol{x}_{\zeta_0}$ in the grid. The feasible realizations of recovery trajectories (i.e. the convergence into $\mathcal{B}(\epsilon)$ before $\zeta_{\rm trans}$) constitute the recoverability bundle\footnote{Here, only positive disturbed state $\dot{x} > 0$ is considered. Recovery of $\dot{x} < 0$ could be achieved in a similar manner. All that is needed is to integrate phase-space trajectories in a backward pattern, detect the PSM deviation by Eq.~(\ref{eq:surface1}) and look up an offline DP policy table designed for backward walking. 
}. 
An example of an estimated recoverability bundle is shown in Fig.~\ref{fig:RecoverabilitySet} (a). 
\subsection{Discrete Foot Placement Control}
\label{subsec:secondoptimization}
When the disturbance is large enough to move the state outside of the recoverability manifold, the controller will not be able to recover to the invariant bundle.  We propose to use the guard strategies discussed in Section~\ref{sec:manifold} for recovery. As a case study, let us use the position guard strategy and re-plan the foot placement for the next step as was illustrated in Fig.~\ref{fig:VelocityDisturbSet1} (c.2). In that strategy, the next apex velocity, $\dot{x}_{{\rm apex}_{q+1}}$, is kept as planned. Hence, we analytically solve for a new foot placement based on the PSM given in Eq.~(\ref{eq:simplifiedPSM}) such that the apex velocity is achieved. Let us define the re-planned phase-space transition state as $(x_{\rm trans}, \dot{x}^{\rm rep}_{\rm trans})$, where only velocity $\dot{x}^{\rm rep}_{\rm trans}$ varies. Since $\dot{x}_{{\rm apex}_{q+1}}$ is unchanged, the adjusted Sagittal foot placement $x_{{\rm foot}_{q+1}}$ is solved by simple manipulation of Eq.~(\ref{eq:simplifiedPSM})
\begin{align}\label{eq:replanfoot}
 x_{{\rm foot}_{q+1}} = x_{\rm trans} + \dfrac{1}{\omega} (\dot{x}^{{\rm rep} 2}_{\rm trans} 
                  - \dot{x}^2_{{\rm apex}_{q+1}})^{1/2}.
\end{align}
For forward walking, $x_{{\rm foot}_{q+1}} > x_{\rm trans}$, prompting us to ignore the solution with the negative square root. 
To evaluate the performance of this step re-planning method, we consider the six disturbances scenarios of Fig.~\ref{fig:High_Level_Robust}. The top three scenarios are recoverable using the DP-based continuous controller that we presented earlier. In the bottom three scenarios the disturbance occurs too close to the transition or is too large and therefore requires the foot placement re-planner described above to be executed. Once foot placements have been re-planned in the Sagittal direction, lateral foot placements are re-planned using Algorithm~\ref{al:newtonraphson}. 

To conclude, the two-stage procedure discussed in this section constitutes the core process of our robust-optimal phase-space planning strategy. The combine locomotion planning procedure is shown in Algorithm~\ref{al:overall-planning} in the Appendix. The computational burden of our technique is minimal because: (I) Once a disturbance is detected, the offline DP policy table is quickly looked up rather than re-computed. (II) If disturbances are too high, Eq.~(\ref{eq:replanfoot}) will quickly yield a new proper foot location, relying on the efficiency of a closed analytical solution.

\section{Simulated Results}
\label{sec:topologies}


\begin{figure*}[t]
 \centering
 \includegraphics[width=\linewidth]{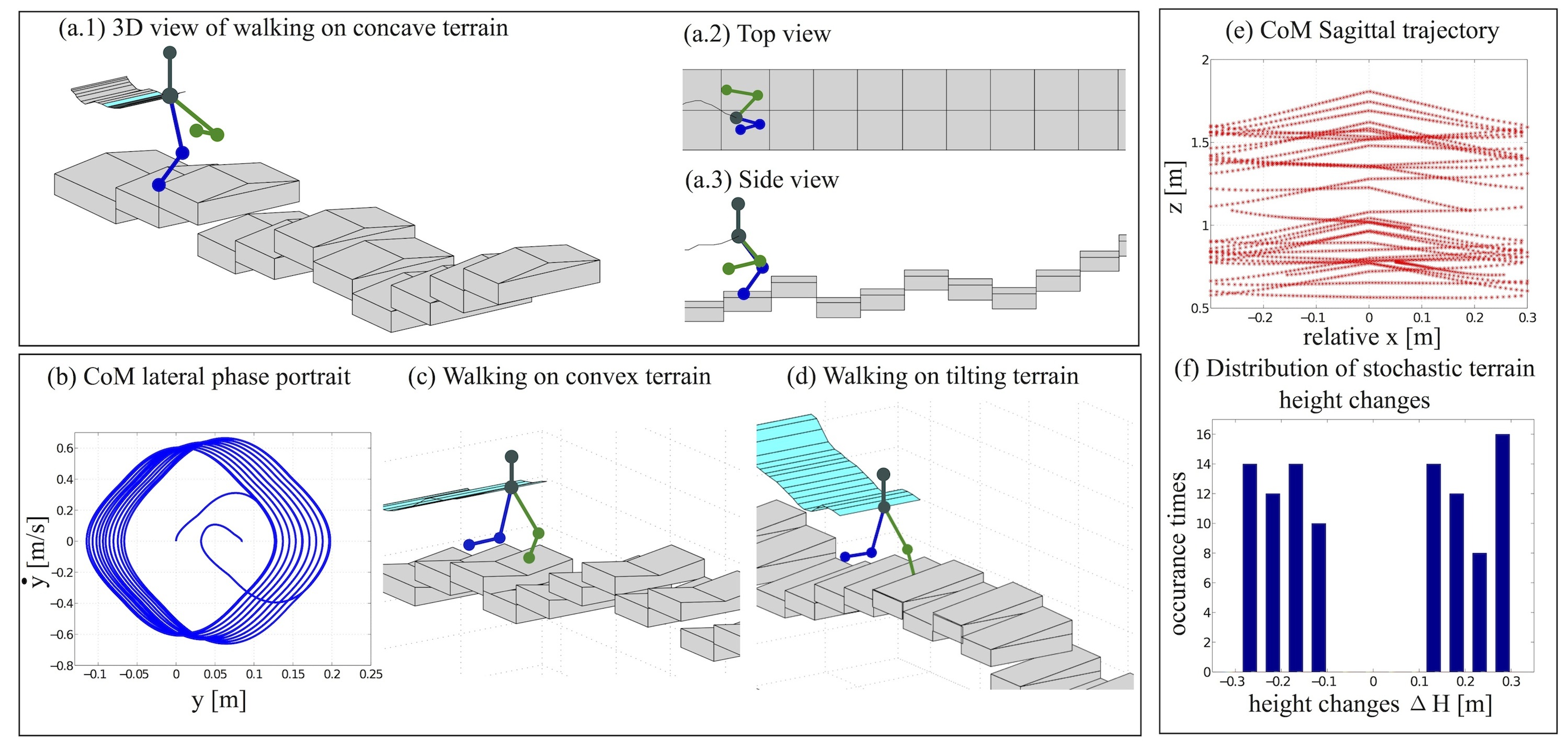}
 \caption{\captionsize Traversing various rough terrains. The subfigures on the left block show dynamic locomotion over rough terrains with varying heights. The block on the right shows the height variation distribution over 100-steps.}
 \label{fig:DifferentTerrain}
\end{figure*}

\subsection{Dynamic Walking over Rough Terrains}
We validate the versatility of our phase-space planning and control strategy by performing locomotion over terrains with height variations randomly generated. Three challenging terrains are tested as shown in Fig.~\ref{fig:DifferentTerrain}: (a) a terrain with concave steps, (c) a terrain with convex steps and (d) a terrain with inclined steps. The height variation, $\Delta h_k$, of two consecutive steps is randomly generated based on the uniform distribution, 
\begin{equation}
\Delta h_k  = h_{k+1} - h_{k} \sim \textrm{U} \left\{
     (-\Delta h_{\rm max}, -\Delta h_{\rm min})\cup
     (\Delta h_{\rm min}, \Delta h_{\rm max})\right\},
\end{equation}
 where, $h_{k}$ represents the height of the $k^{\rm th}$ step, $\Delta h_{\rm min} = 0.1$ m, $\Delta h_{\rm max} = 0.3$ m. A $10^\circ$ tilt angle is used for the slope of the steps. Foot placements are chosen a priori using simple kinematic rules and considering the length of the terrain steps. Separately, we choose apex velocities to be around $0.6$ m/s, plus minus a small increment depending on the height change. Finally we choose simple CoM path manifolds that conform to the steps. We then apply the planning pipeline outlined in previous sections including, the generation of trajectories based on Algorithm 1 and the search for step transitions based on the procedures of Section 5.2. 

Fig.~\ref{fig:DifferentTerrain} (a) shows a snapshot of bipedal walking on the terrain with concave steps.  The lateral CoM phase portrait in Fig.~\ref{fig:DifferentTerrain} (b) shows stable walking over 25 steps.  Fig.~\ref{fig:DifferentTerrain} (c) and (d) show other types of rough terrains also tested in simulation over 100 steps.   Fig.~\ref{fig:DifferentTerrain} (e) shows the CoM path manifolds we choose to step over the terrain.  The bar graph in Fig.~\ref{fig:DifferentTerrain} (f) shows the distribution of the height of the randomly generated terrain. 

\begin{figure*}[t]
 \centering
   \includegraphics[width=0.75\linewidth]{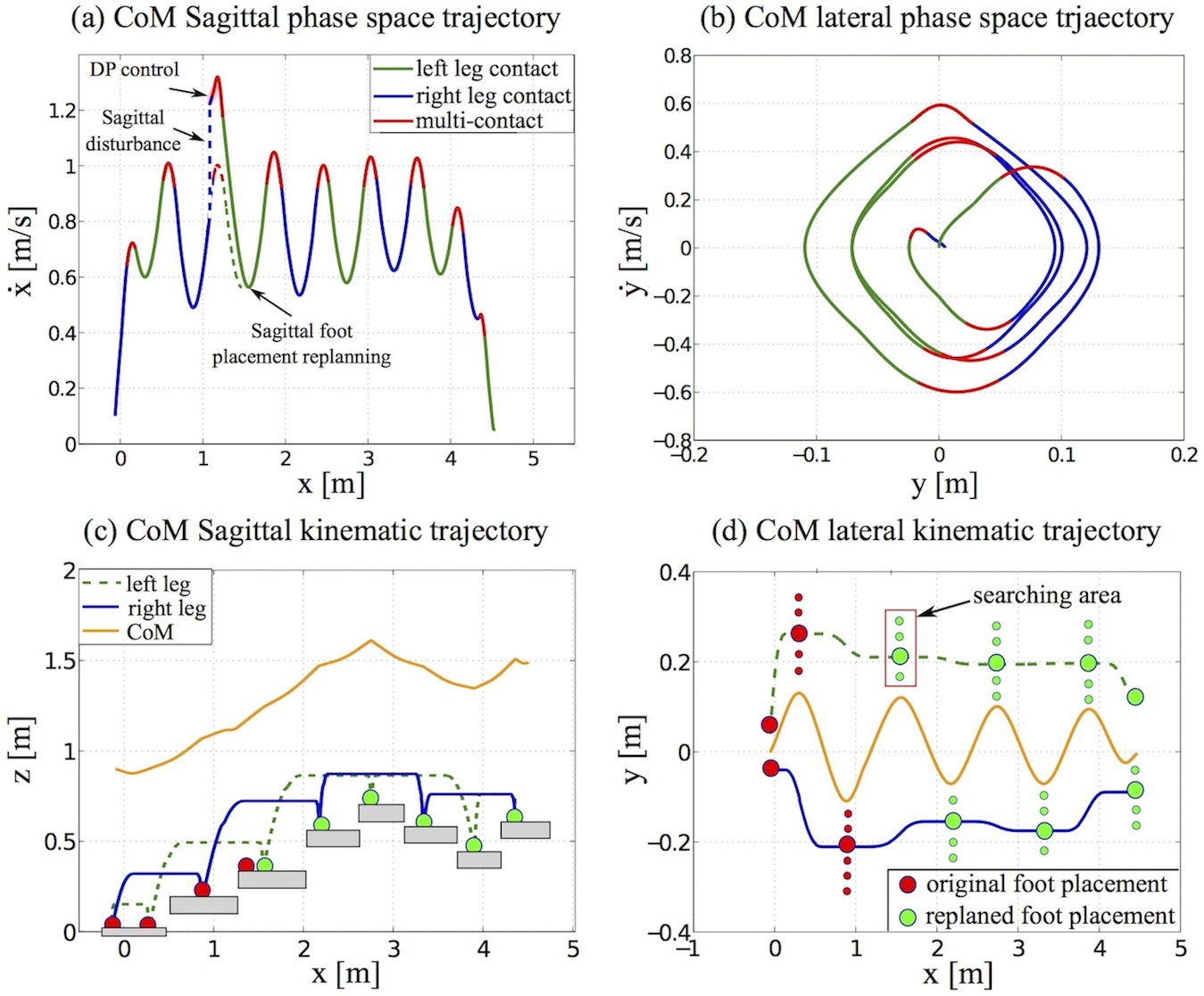}
 \caption{\captionsize Recovery from Sagittal plane disturbance. To recover from a Sagittal push, the planner uses both DP continuous control and discrete foot placement re-planning in a sequential manner. $\reddot$ denotes the pre-defined foot placement before the disturbance while $\greendot$ denotes the re-planned foot placement after the disturbance.
} 
 \label{fig:SagittalDisturbance}
 \vspace{-1mm}
\end{figure*}

\subsection{Dynamic Walking under External Disturbances}


\subsubsection{Recovery from Disturbance on the Sagittal Plane}

We first make the robot walk on a terrain based on the planning algorithms described in the paper. We then apply a pushing force in the Sagittal direction, which causes an instantaneous velocity jump as shown in Fig.~\ref{fig:SagittalDisturbance} (a). This disturbance is quite large such that the robot's state cannot recover to its nominal PSM using the proposed optimal controller. Thus, the foot location re-planning strategy previously described is executed. The dashed line in Fig.~\ref{fig:SagittalDisturbance} (a) represents the original phase-space trajectory while the solid line represents the re-planned trajectory. Also, instead of an instantaneous step transition, a multi-contact transition is used as described in the Appendix~\ref{sec:multicontact}.

\begin{figure*}[t]
 \centering
   \includegraphics[width=0.8\linewidth]{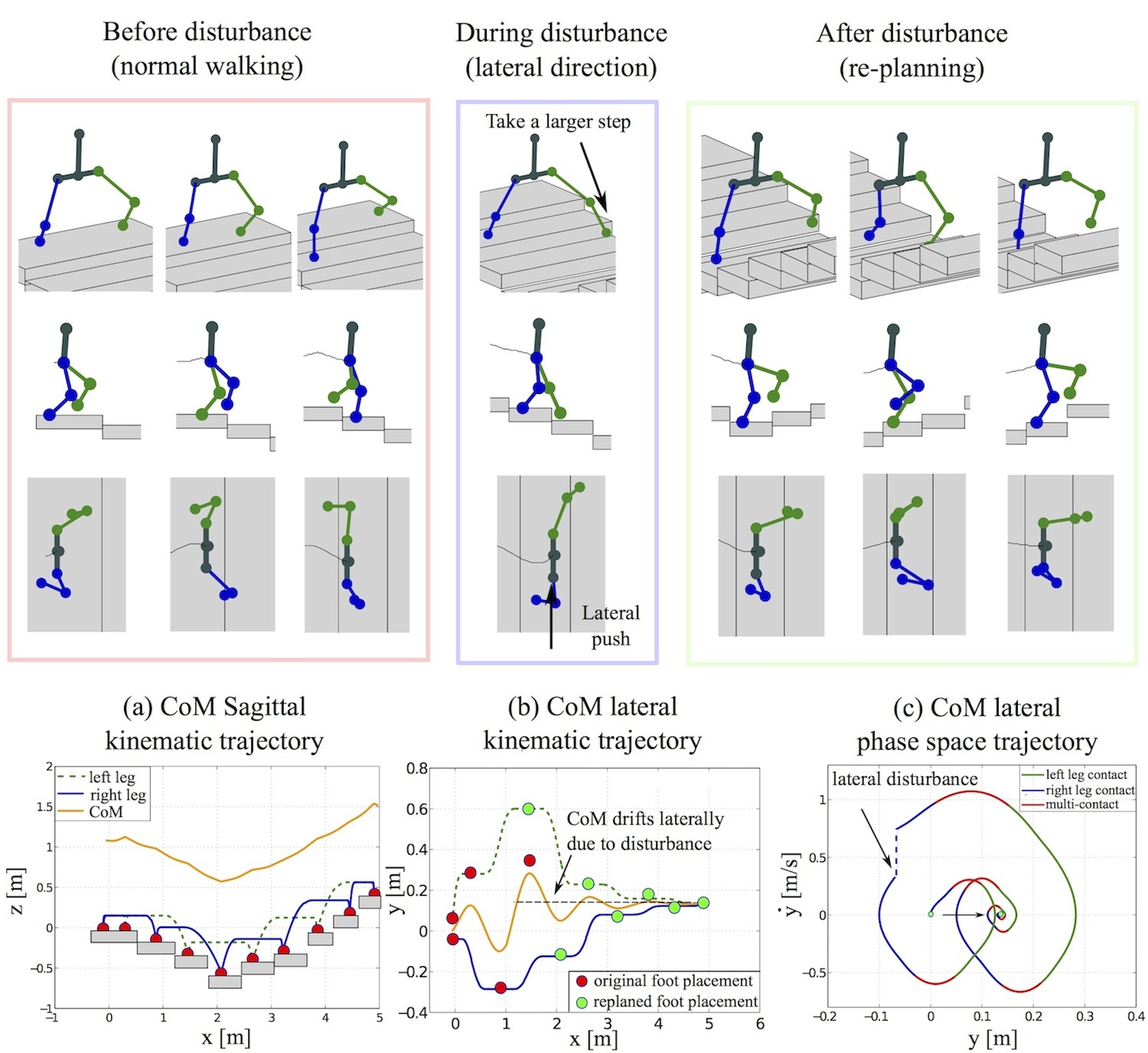}
 \caption{\captionsize Recovery from lateral plane disturbance. When the lateral disturbance occurs, the foot placement is re-planned to avoid falling down. Correspondingly, the CoM trajectory drifts to the left side as shown in (b) while a CoM velocity jump appears in (c). After the disturbance, the CoM trajectory is re-generated based on this new lateral foot placement to achieve stable walking.}
 \label{fig:SequantialSnapshot}
\end{figure*}

\subsubsection{Recovery from Disturbance on the Lateral Plane}

For this simulation, we make the robot once more walk on the rough terrain according to its nominal plan. Then, in its third step, we apply a lateral disturbance as shown in Fig.~\ref{fig:SequantialSnapshot} (b) and (c). To deal with this disturbance, a new lateral foot placement is re-planned according to Algorithm 2. 

\subsection{Bouncing over A Disjointed Terrain}
A more challenging locomotion scenario is explored using a disjointed terrain. The slope of the surfaces is $70^\circ$. The goal is to step up over the surfaces by bouncing over the terrain. A physics based dynamic simulation called SrLib is used for validation and a whole body operational space controller [\cite{sentis2010compliant}] is implemented to follow the locomotion plans.

\begin{figure*}[t]
 \centering
   \includegraphics[width=0.85\linewidth]{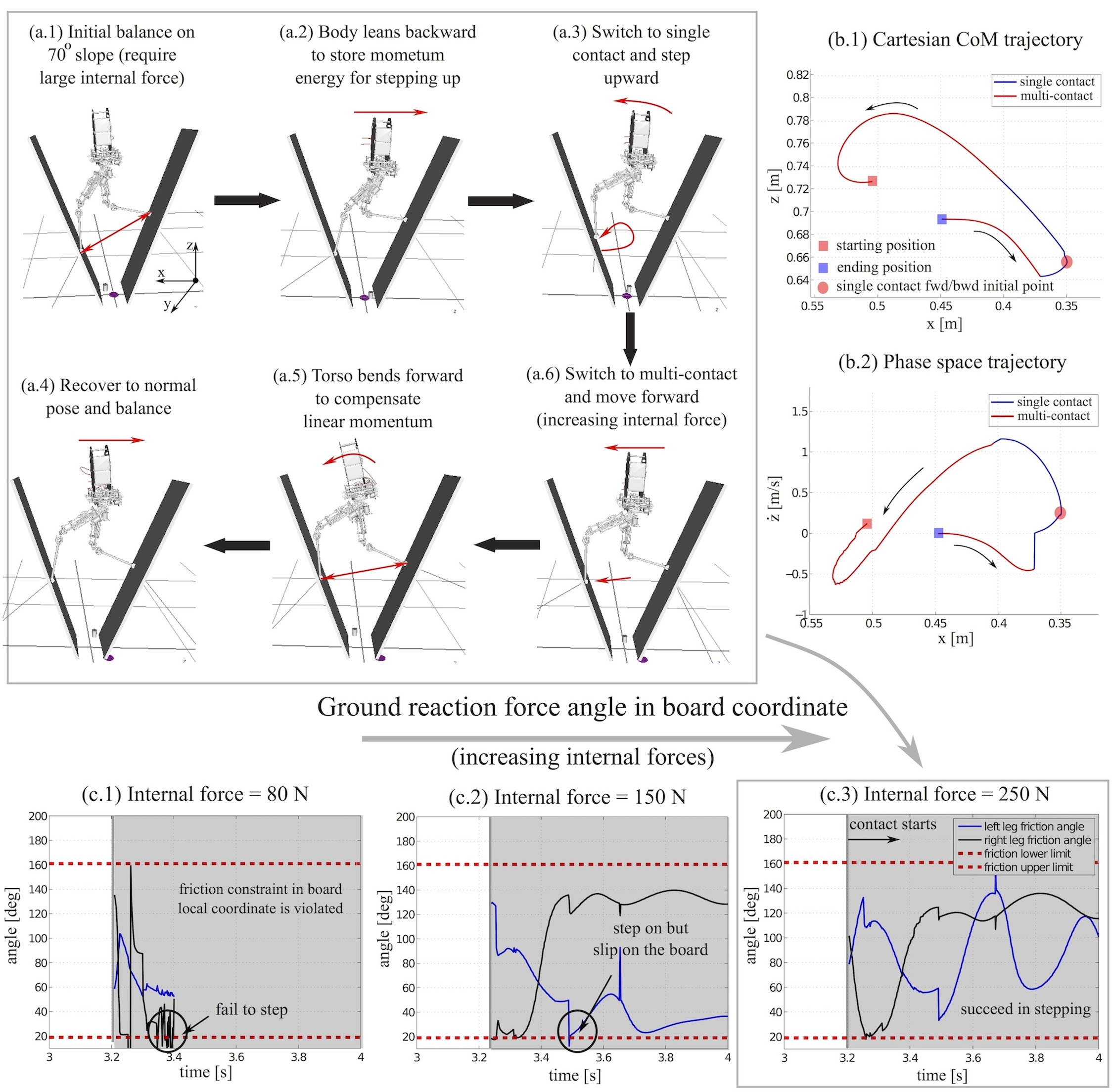}
 \caption{\captionsize Bouncing over a disjointed terrain. In this simulation, a biped balances on a steep disjointed terrain and dynamically bounces upwards. Internal force and angular momentum are controlled properly to achieve this motion. Subfigure (b) shows the 2D Cartesian CoM trajectory and the $x-\dot{z}$ phase portrait. (c) shows contact reactions for three different desired internal forces.}
 \label{fig:wedgejumping}
\end{figure*}

Snapshots of a one-step bouncing behavior are shown in Fig.~\ref{fig:wedgejumping} (a). To successfully bounce over the terrain, we use a CoM path manifold, shown in Fig.~\ref{fig:wedgejumping} (b.1), that mimics that of a pre-recorded human [\cite{sentis2011humanoids}]. During the multi-contact phase, we apply a $250$ N internal tension force, shown in Fig.~\ref{fig:wedgejumping} (c), between the two surfaces to avoid sliding down due to the weight of the robot. Our planner operates in the $x-\dot z$ state as shown in Fig.~\ref{fig:wedgejumping} (b.2). This is more convenient as the $\dot z$ reveals the moment at which the center of mass starts falling down. More details on this strategy are discussed in [\cite{sentis2011humanoids, zhao15hybrid}]. Note that an apex state is not used in the same way that we defined for previous rough terrain locomotion tests. Instead, we define a priori a desired keyframe state, shown as a red circle in Fig.~\ref{fig:wedgejumping} (b.2).

\section{Discussions and Conclusions}
\label{sec:Discussion}

The main focus of this paper has been on addressing the needs for planning non-periodic bipedal locomotion behaviors. These types of behaviors arise in situations where terrains are non-flat, disjointed, or extremely cluttered. The majority of bipedal locomotion frameworks have been historically focused on flat terrain or mildly rough terrain locomotion behaviors. Some of them are making their way into dynamically climbing stairs or inclined terrains. Additionally, Raibert accomplished hopping locomotion over rough terrains in the middle-80s. In contrast, our effort is centered around the goals of (1) providing metrics of robustness in rough terrain for robust control of the locomotion behaviors, (2) generalizing gaits to any types of surfaces including disjointed terrains, (3) providing formal tools to study planning, robustness, and reachability of the non-periodic gaits, and (4) demonstrating the ability of our framework to deal with large external disturbances.

Future work will focus on: (i) Experimental validations of the proposed planning and control strategy, where modeling, pose estimation, and kinematic errors, among other problems will greatly impact performance. (ii) Proposing a more realistic robot model that incorporates swing leg dynamics, and yaw and roll angular momentum. (iii) Extending the proposed framework for planning and control of omnidirectional dynamic locomotion in rough terrains. (iv) Proposing a realistic terrain perception model that does not assume perfect information about the terrain structure.

 \begin{funding}
 This work was supported by the Office of Naval Research, ONR Grant [grant \#N000141210663], NASA Johnson Space Center, NSF/NASA NRI Grant [grant \#NNX12AM03G], and NSF CPS Synergy Grant [grant \#1239136].
 \end{funding}

\begin{appendix_sec}
\appendix
\label{appen}

\section{Index to Multimedia Extensions}

\begin{tabular}{l*{6}{c}r}
Extension  & Type & Description \\
\hline
& & Demonstrations of four locomotion simulations including (1) seven-step  \\
1 & Video & phase-space motion planning; (2) dynamic walking over rough terrains; (3) dynamic \\
& & walking under external disturbances; (4) bouncing maneuver over disjointed terrain.\\
\hline
\end{tabular}

\section{Phase-Space Manifold}
\label{sec:PSMBasics}

The desired behavior of the outputs lie in manifolds $\mathcal{M}_i$ as shown in Eq. (\ref{eq:surface}).  Here we present a brief review of space curves and surfaces that relate to the \cool{phase-space manifold} (\cool{PSM}) and present a Riemannian geometry metric that can be generalized to this family of problems. 
 
The trajectory of the center of mass (CoM) of the robot is a space curve in 3D, $\boldsymbol{p}_{\rm CoM}=\{x,y,z\}\onR{3}$.  Also, for a particular output $y_i$, if we consider the case of an output-task with relative order $r_i=3$, the manifold $\mathcal{M}_i\onR{r_i}$ is a space curve $\mathcal{C}_i$ in Euclidean three-dimensional space (see Fig. \ref{fig:spacecurve}).   We assume that the curve is parametrized by an arc-length parameter $\zeta$ that we refer to as the \cools{progression} variable. Hence the position vector $\boldsymbol{\rho}_i$ of any point on the curve can be defined by specifying the value of $\zeta$,
\begin{equation}  \label{(2)}
 \boldsymbol{\rho}_i(\zeta) = \sum_{k=1}^{r_i}\xi_k(\zeta){\bf E}_k
           = \xi_1(\zeta){\bf E}_1 + \xi_2(\zeta){\bf E}_2 + \xi_3(\zeta){\bf E}_3. 
\end{equation}
\begin{figure*}[t]
 \centering
 \includegraphics[width=0.6\linewidth]{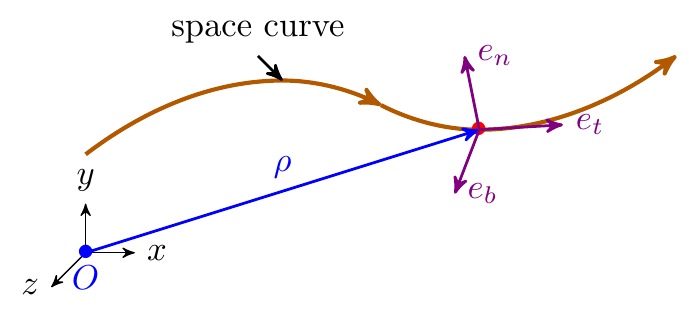}
 \caption{\captionsize A space curve showing the evolution of the Frenet triad.}
 \label{fig:spacecurve}
\end{figure*}
where, ${\bf E}_k$ is the unit vector in the $k$-axis of the Euclidean space and $\xi_k(\zeta)$ is the projection coordinate of $\boldsymbol{\rho}$ on ${\bf E}_k$. A unit tangent vector ${\bf e}_t$ to the curve can also be defined,
%
\begin{equation} \label{(3)}
 {\bf e}_t = \dfrac{\partial \boldsymbol{\rho}_i}{\partial \zeta}. 
 \end{equation}
The derivative of this vector defines the curvature $\kappa$ and the unit normal vector ${\bf e}_n$,
%
\begin{equation} \label{(4)} 
 \dfrac{\partial {\bf e}_t}{\partial \zeta} = \kappa {\bf e}_n, 
\quad {\rm where} \quad
    \kappa = \left|\!\left| \dfrac{\partial {\bf e}_t}{\partial \zeta}\right|\!\right|. 
\end{equation}
In the case of a space surface (where, $\boldsymbol{\rho}_i$ belongs to a manifold $\mathcal{M}_i$ in the output phase-space), instead of a vector ${\bf e}_t$, we have a tangent  manifold, denoted by $T_{\mathcal{M}_i}$.  The tangent space at any point can be mapped to the vector $\boldsymbol{\chi}_i\in \mathbb{R}^{r_i-1}$ that spans $T_{\mathcal{M}_i}$.  Without loss of generality, the actual motion is a specific line in space curve $\mathcal{M}_i$.  The tangent vector in the manifold $\mathcal{M}_i$ is $\boldsymbol{e}_\zeta=({\bf e}_i)_t$ while the cotangent vector in the manifold is ${\bf e}_\sigma=({\bf e}_i)_n$. ${\bf e}_\sigma$ denotes the normal deviation distance from the surface $\sigma_i$. For $r_i=3$, the binormal vector, ${\bf e}_b$ is orthogonal to ${\bf e}_t$ and ${\bf e}_n$. These three vectors are called the Frenet space. These three Frenet frame vectors are proportional to the first three derivatives of the curve $\boldsymbol{\rho}$, as a benefit of taking the arc length $\zeta$ as the parameter.

In disturbance-free cases, the system will remain in the manifold if it starts on it.  It can be considered as the zero dynamics of the surface deviation, $\sigma_i$.  When disturbance occurs, the state may escape the manifold and the controller should bring it back for recovery. To define a metric on the manifold itself and normal to it, we use Riemannian Geometry. In general, we treat each manifold $\mathcal{M}_i$ in Eq.~(\ref{eq:surface}) of the task-space $i^{\rm th}$-coordinate as independent from each other. The actual task manifold, $\mathcal{M}$ is the intersection of all $\mathcal{M}_i$ manifolds,
\begin{equation} \label{(7)}
 \mathcal{M} = \bigcap_{i}^{} \mathcal{M}_i. 
\end{equation}
This manifold also has a tangent manifold $T_{\mathcal{M}} \in \mathbb{R}^r$, where $r=\sum_i{r_i}$.  Each manifold $\mathcal{M}_i$, separately have a null-space (cotangent manifold, $T^{*}_{\mathcal{M}_i}$) and their intersection is the task cotangent manifold, $T^{*}_{\mathcal{M}}$. 

\begin{figure*}[t]
 \centering
 \includegraphics[width=0.95\linewidth]{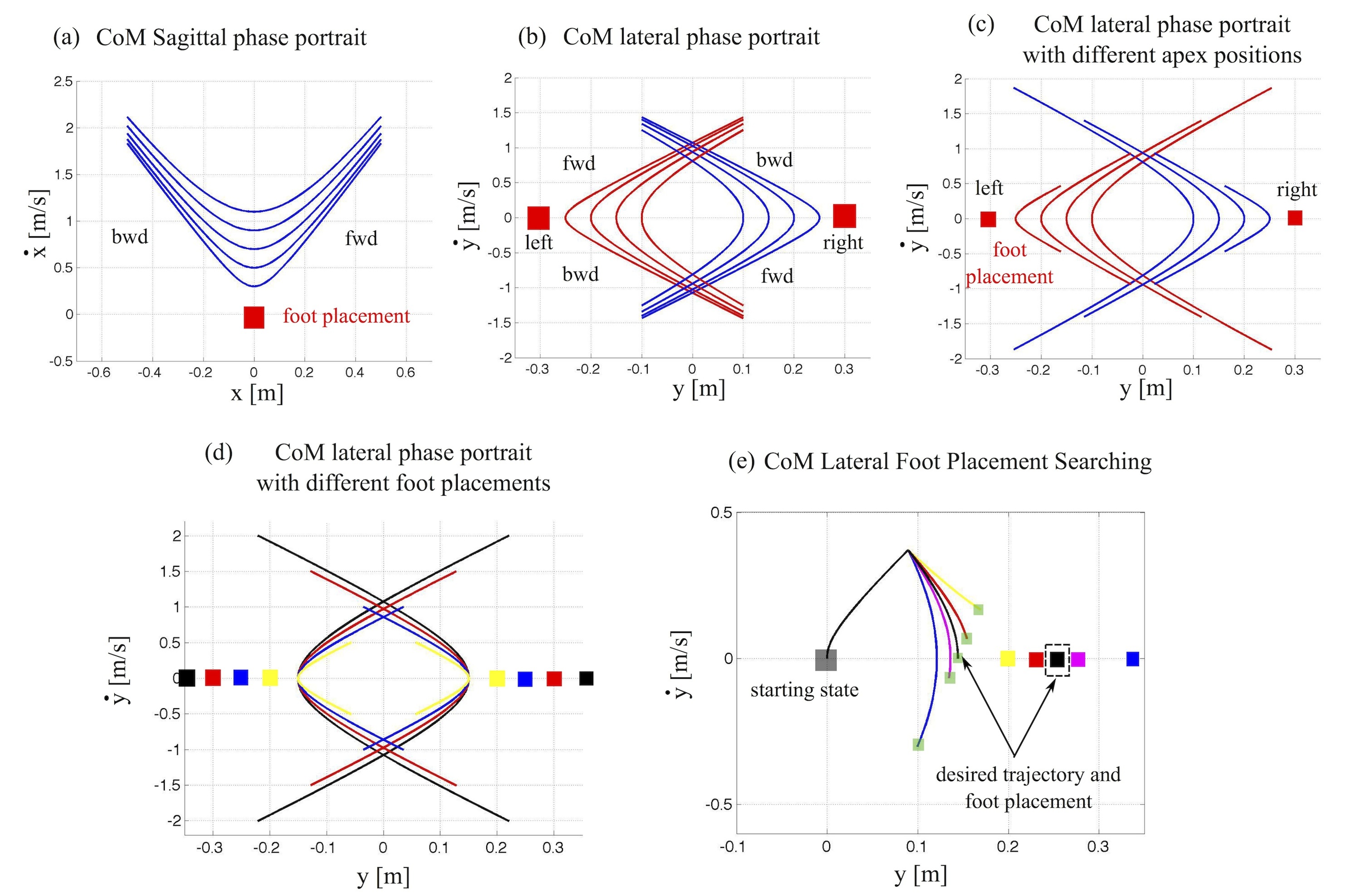}
 \caption{\captionsize Numerical integration of Sagittal and lateral dynamics. Numerical Integration is used for phase-space prediction of Sagittal and lateral prismatic
   inverted pendulum dynamics in Eq.~(\ref{eq:accel}). The phase portraits (a) and (b) correspond to the Sagittal and
   lateral CoM phase behaviors given desired foot contact
   locations (red boxes), a desired CoM surface of motion,
   and initial position and velocity conditions. If the same time duration is guaranteed for each trajectory, we can derive two different trajectories shown in (c) and (d) based on different initial conditions.  (c) shows lateral CoM behaviors given a fixed lateral foot placement and varying initial lateral position conditions. (d) corresponds to CoM
   trajectories derived given varying lateral foot placements and a fixed
   initial condition. In (e), we analyze lateral CoM trajectories of two consecutive steps with varying lateral foot placements. As the foot placement moves further apart, the acceleration becomes larger and the CoM position transverses (at the Sagittal apex) less in the y direction.}
 \label{fig:Numerical_Integration}
\end{figure*}

\section{Numerical Integration}
\label{sec:numInteg}

For the nonlinear Eq.~(\ref{eq:accel}), we assume that $\ddot x$ is approximately constant for small
increments of time. Since $f$ is normally highly nonlinear, numerical integration algorithms are normally used. More details are explained in [\cite{sentis2011humanoids, zhao2013robust}]. The pipeline for finding state-space trajectories goes as follows: (1) choose a small time perturbation $\epsilon$, (2)
given known velocities $\dot x_k$ and accelerations $\ddot x_k$, we derive the next velocity $\dot x_{k+1}$, and the next position
$x_{k+1}$, (3) plot the point $(x_{k+1},\, \dot x_{k+1})$ in the
phase-plane. We can iterate this recursion both
forward and backward. Based on this pipeline, one-step Sagittal and lateral manifold with different initial conditions are shown in Fig.~\ref{fig:Numerical_Integration}.

\section{Multi-Contact Maneuvers}
\label{sec:multicontact}

The objective of this section is to incorporate multi-contact transitions into our gait planner to achieve more natural motions. Towards this goal, we fit a polynomial function with the desired smooth behavior to the phase curve in the vicinity of the step transition. For this process, desired boundary values of position, velocity and acceleration are given by the gait designer. It is necessary to also take into account time constraints to guarantee the synchronization of the Sagittal and lateral behaviors. Boundary and timing conditions allow us to calculate the coefficients of the polynomials. More mathematical details of this approach can be found in [\cite{zhao2012three}]. In our case, a multi-contact phase is created to occupy $25\%$ of the total time slot for a given step. This percentage is adjustable based on the desired walking profile. The results are illustrated in Fig.~\ref{fig:Multicontact}. 

\begin{figure}[t]
 \centering
   \includegraphics[width=0.9\linewidth]{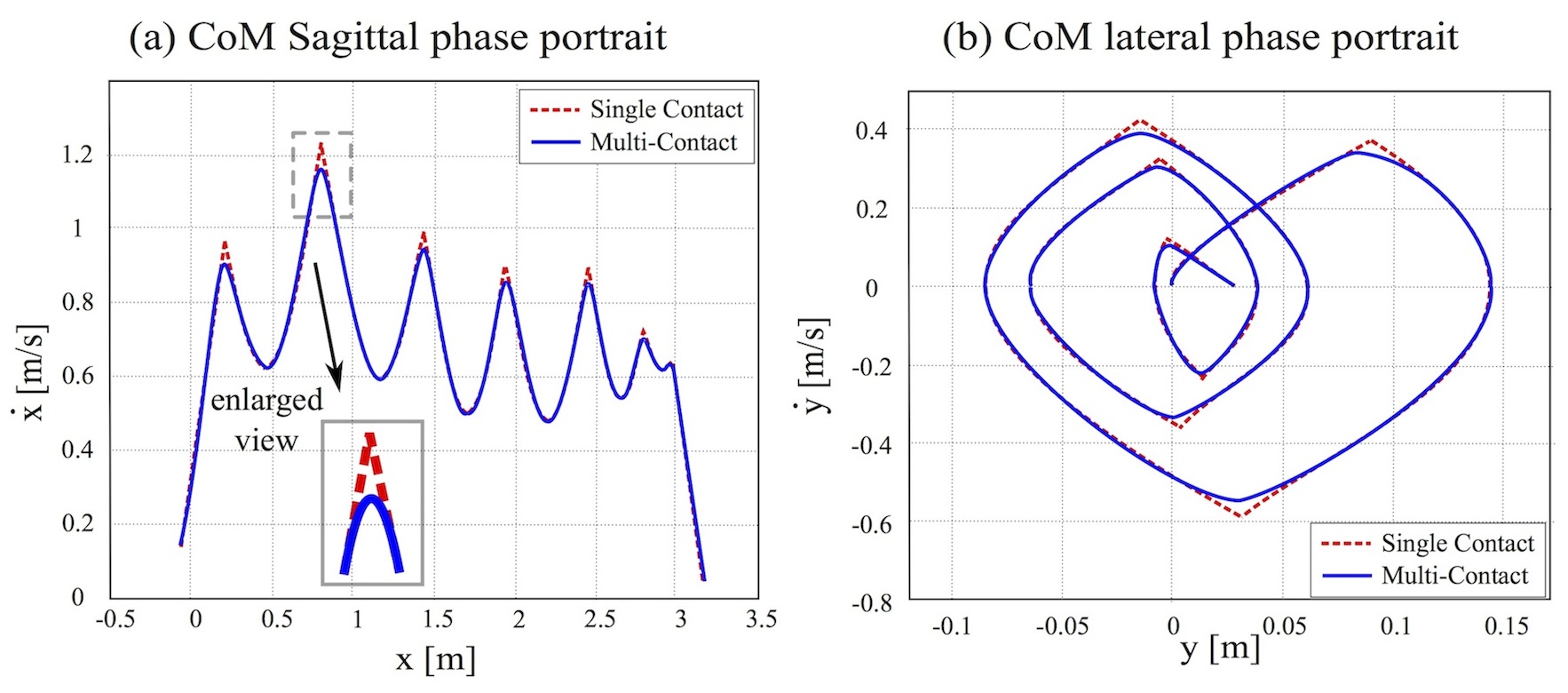}
 \caption{\captionsize Integration of multi-contact transition phases. The subfigures (a) and (b) are similar to their counterparts of Fig.~\ref{fig:3Dsinglecontact} but with an additional multi-contact phase. By using $5^{\rm th}$ order polynomials and guaranteeing continuity with the existing curves, we get the polynomial parameters for curve fitting.}
 \label{fig:Multicontact}
\end{figure}

\section{Proof of Phase-Space Tangent Manifold}
\label{sec:PSMDerivation}

In the nominal control case, the angular momentum $\tau_y$ of Eq.~(\ref{eq:accel}) is zero. For this case, the Sagittal inverted pendulum dynamics are simple, $\ddot{x} = \omega^2(x - p_x)$, where $\omega$ is constant for a given step. Since the foot placement $p_x$ is also constant over the step, then $\ddot{p}_x = \dot{p}_x = 0$. Therefore the previous equation is equivalent to $\ddot{x} - \ddot{p}_x = \omega^2 (x - p_x)$. Let us define a transformation $\tilde{x} = x - p_x$. We can then write $\ddot{\tilde{x}} = \omega^2 \tilde{x}$. Using Laplace transformations, we have $s^2 \tilde{x}(s) - \tilde{x}_0 - s \dot{\tilde{x}}_0 = \omega^2 \tilde{x}(s)$. Based on this, we get
\begin{equation}
\tilde{x}(t) = \mathscr{L}^{-1}\{\dfrac{\tilde{x}_0 + s\dot{\tilde{x}}_0}{s^2 - \omega^2}\}.
\end{equation}
Solving the equation above, we can derive an analytical solution
\begin{align}\label{eq:acceleration3}
\tilde{x}(t) = \dfrac{\tilde{x}_0 (e^{\omega t} + e^{-\omega t})}{2} + \dfrac{\dot{\tilde{x}}_0
(e^{\omega t} - e^{-\omega t})}{2 \omega} = \tilde{x}_0 {\rm cosh}(\omega t) + \dfrac{1}{\omega} \dot{\tilde{x}}_0 {\rm sinh}(\omega t),
\end{align}
and by taking its derivative, we get
\begin{align}\label{eq:acceleration3-0}
\dot{\tilde{x}}(t) = \omega \tilde{x}_0 {\rm sinh}(\omega t) + \dot{\tilde{x}}_0 {\rm cosh}(\omega t).
\end{align}
These two equations can be further expressed as
\begin{eqnarray}\label{eq: closedformposition}
x(t) &=& (x_0 - p_x) {\rm cosh}(\omega t) + \dfrac{1}{\omega} \dot{x}_0 {\rm sinh}(\omega t) + p_x
\\
\label{eq: closedformvelocity}
\dot{x}(t) &=& \omega (x_0 - p_x) {\rm sinh}(\omega t) + \dot{x}_0 {\rm cosh}(\omega t)
\end{eqnarray}
Now we have the following state space formulation
\begin{eqnarray}\nonumber
\begin{pmatrix}
x(t) - p_x\\
\dot{x}(t)\\
\end{pmatrix}
&=& \begin{pmatrix}
x_0 - p_x & \dot{x}_0 / \omega\\
\dot{x}_0 & \omega (x_0 - p_x)\\
\end{pmatrix}
\begin{pmatrix}
{\rm cosh}(\omega t)\\
{\rm sinh}(\omega t)\\
\end{pmatrix} \qquad
\end{eqnarray}\nonumber
which implies
\begin{eqnarray}\label{eq: sinhcosh}
\begin{pmatrix}
{\rm cosh}(\omega t)\\
{\rm sinh}(\omega t)\\
\end{pmatrix}
&=& \dfrac{1}{\omega (x_0 - p_x)^2 - \dot{x}^2_0/\omega}
\begin{pmatrix}
\omega (x_0 - p_x) & -\dot{x}_0 / \omega\\
-\dot{x}_0 & x_0 - p_x\\
\end{pmatrix}
\begin{pmatrix}
x - p_x\\
\dot{x}\\
\end{pmatrix}
\end{eqnarray}
since ${\rm cosh}^2(x) - {\rm sinh}^2(x) = 1$, we get 
\begin{align}\label{eq:surface0}
\big(\omega(x_0 - p_x) (x - p_x) - \dot{x}_0 \dot{x}/\omega\big)^2 - \big(- \dot{x}_0 (x - p_x) 
 + \dot{x} (x_0 - p_x)\big)^2 = \big(\omega (x_0 - p_x)^2 - \dot{x}^2_0/\omega\big)^2
\end{align}
After expanding the square terms and moving all terms to one side, we obtain 
\begin{align}
(x_0 - p_x) ^2\big(2\dot{x}^2_0 - \dot{x}^2 + \omega^2 (x - x_0) (x + x_0 - 2p_x)\big) - \dot{x}^2_0 (x - p_x)^2 + \dot{x}^2_0 (\dot{x}^2 - \dot{x}^2_0)/\omega^2 = 0
\end{align}
which is the phase-space tangent manifold $\sigma$ defined in Proposition~\ref{theorem:PSM}. $\qquad \qquad\qquad \qquad\qquad\qquad\qquad\qquad\qquad\qquad \qed$

\setcounter{algorithm}{2}
\begin{algorithm}[t]
\begin{algorithmic}[1]
\setstretch{1.1}
\STATE Initialize walking step index $k \gets 1$, discrete state $q$, initial condition $\mathcal{I}_{q}$, $\epsilon$ for invariant bundle $\mathcal{B}_{q}(\epsilon)$, 
stage update indicator $b_{\rm update} \gets \FALSE$.
\WHILE{$\boldsymbol{x}_{q} \notin \mathcal{G}_a^{[\delta_j]}(q, q+1)$}
\IF{$\boldsymbol{x}_{q} \in \mathcal{G}_d^{[\delta_j]}(q, q_{\rm dist})$}
\STATE Execute $\Delta_d^{[\delta_j]}(q, q_{\rm dist}, \boldsymbol{x}_{q_{\rm dist}})$ and quantize disturbed state $(\boldsymbol{x}_{q})_{\rm dist}$.
\STATE Generate optimal policies $(\zeta, \dot{x}, \ddot{x}, \tau, \omega, \mathcal{L}, \mathcal{V})_{\rm opt}$ by dynamic programming.
\STATE $b_{\rm update} \gets \TRUE$.
\STATE Compute the phase-space manifold $\sigma_{\rm trans}$ by Eq.~(\ref{eq:surface1}) at transition phase.
\IF{$(\boldsymbol{x}_{q})_{\rm trans} \notin \mathcal{R}_{q}$}
\STATE Re-plan $x_{{\rm foot}_{q+1}}$ by Eq.~(\ref{eq:replanfoot}) and search $y_{{\rm foot}_{q+1}}$ by Algorithm~2.
\ENDIF
\ENDIF
\STATE Compute $\sigma_{i+1}$ over domain $\mathcal{D}_{q}$ by Eq.~(\ref{eq:surface1}).
\IF{$b_{\rm update}$ is $\TRUE$}
\STATE Update stage index $i_{\rm stage}$ of recovery optimal control inputs.
\ENDIF
\IF{$\boldsymbol{x}_{q} \notin \mathcal{B}_{q}$}
\STATE Compute $\boldsymbol{u}_{c_{i+1}} =: (\tau_y, \omega)_{i_{\rm stage}}$ by Eq.~(\ref{eq:slidingcontrol_a}) and assign $\ddot{x}_{i+1} \gets \ddot{x}_{\rm opt}(i_{\rm stage})$.
\ELSE 
\STATE Compute $\boldsymbol{u}_{c_{i+1}} =: (\tau_y, \omega)_{i+1}$ by Eq.~(\ref{eq:slidingcontrol_b}) and assign $\ddot{x}_{i+1}$ by Eq.~(\ref{eq:accel}).
\ENDIF
\STATE Evolve $(x_{i+1}, \dot{x}_{i+1})$ over domain $\mathcal{D}_{q}$ numerically.
\STATE $i \gets i + 1$.
\ENDWHILE
\STATE $q \gets q+1$, re-assign $\mathcal{I}_{q+1}$, $b_{\rm update} \gets \FALSE$ and jump to line 2 for next walking step.
\end{algorithmic}
\caption{Overall Hybrid Robust Locomotion Planning Structure}
\label{al:overall-planning}
\end{algorithm}

\section{Proof of Phase-Space Cotangent Manifold}
\label{sec:OrthogonalManifoldDerivation}
In this case, we use the tangent manifold in Eq.~(\ref{eq:simplifiedPSM}) to derive the cotangent manifold $\zeta$. By taking the derivative of Eq.~(\ref{eq:simplifiedPSM}), we have
\begin{align}\label{eq:d_sigma}
d \sigma & = \dfrac{\partial \sigma}{\partial x} dx 
           + \dfrac{\partial \sigma}{\partial \dot{x}} d \dot{x},
\end{align}           
where
\begin{align}
            \dfrac{\partial \sigma}{\partial x}
           = -2 \dot{x}_{\rm apex}^2 (x - x_{\rm foot}), 
           \quad
             \dfrac{\partial \sigma}{\partial \dot{x}}
           =  2 \dot{x}_{\rm apex}^2 \dot{x}/\omega^2.
\end{align}
The $\sigma$ manifold's normal vector is given by its gradient, 
 ${\bf e}_n=\big(-2 \dot{x}_{\rm apex}^2 (x-x_{\rm foot})\;,
                  \;2 \dot{x}_{\rm apex}^2\dot{x}/\omega^2 \big)^T$, 
and its tangent vector is orthogonal to ${\bf e}_n$, i.e.,  
 ${\bf e}_t=\big(2\dot{x}_{\rm apex}^2 \dot{x}/\omega^2\;,\; 
                 2\dot{x}_{\rm apex}^2 (x-x_{\rm foot})\big)^T$. 
 Since $\zeta$ is orthogonal to $\sigma$, the tangent vector of $\zeta$ is the normal vector of $\sigma$, i.e., 
\begin{align}\label{eq:d_zeta}
d \zeta & = \dfrac{\partial \zeta}{\partial x} dx 
          + \dfrac{\partial \zeta}{\partial \dot{x}} d \dot{x},
\end{align}
where
\begin{align}
            \dfrac{\partial \zeta}{\partial x} = 2 \dot{x}_{\rm apex}^2 \dot{x}/\omega^2, 
          \quad 
            \dfrac{\partial \zeta}{\partial \dot{x}}  = -2 \dot{x}_{\rm apex}^2 (x - x_{\rm foot})
\end{align}
Via the equations above, we can further obtain
\begin{align}\label{eq:dxdy}
\dfrac{d \dot{x}}{d x} = -\dfrac{\dot{x}}{\omega^2 (x - x_{\rm foot})} 
                 \qquad \Rightarrow \qquad 
                 \omega^2 \int_{\dot{x}_0}^{\dot{x}}\dfrac{d\dot{x}}{\dot{x}} 
                 = -\int_{x_0}^{x} \dfrac{dx}{x-x_{\rm foot}}
\end{align}
then we have
\begin{align}\label{eq:ln}
\ln(\dfrac{\dot{x}}{\dot{x}_0})^{\omega^2} + \ln \dfrac{x - x_{\rm foot}}{x_0 - x_{\rm foot}} = 0
\qquad \Rightarrow \qquad 
(\dfrac{\dot{x}}{\dot{x}_0})^{\omega^2} \dfrac{x - x_{\rm foot}}{x_0 - x_{\rm foot}} = 1
\end{align}
Thus, the cotangent manifold can be defined as 
\begin{align}\label{eq:cotangent-manifold}
\zeta  = \zeta_0(\dfrac{\dot{x}}{\dot{x}_0})^{\omega^2} \dfrac{x - x_{\rm foot}}{x_0 - x_{\rm foot}}
\end{align}
where the constant $\zeta_0$ is a nonnegative scaling factor. $(x_0, \dot{x}_0)$ is the initial condition at $\zeta=\zeta_0$. The equation above is the phase-space cotangent manifold $\zeta$ defined in Proposition~\ref{prop:PSCoM}. $\qquad\qquad\qquad\qquad\qquad\qquad\qquad\qquad\qquad\qquad\qquad\qed$

\end{appendix_sec}

\bibliographystyle{chicago}
\bibliography{IJRR}

\end{document}